%% file: main.tex
\documentclass{article} 
\usepackage{iclr2024_conference,times}

\input{math_commands.tex}

\usepackage{url}

\usepackage{wrapfig}
\usepackage{amsmath,amssymb}
\usepackage{graphicx} 
\usepackage{subcaption}
\usepackage{amsmath}
\usepackage{amssymb}
\usepackage{booktabs}
\usepackage{multirow}
\usepackage{bbm}
\usepackage{xcolor}
\definecolor{codeblue}{rgb}{0.25,0.5,0.5}
\usepackage[linesnumbered,ruled,vlined]{algorithm2e}
\usepackage{comment}
\usepackage{array}

\newcommand{\Eall}{\Ecal_\textnormal{all}}
\newcommand{\Eava}{\Ecal_\textnormal{avail}}
\SetKwComment{Comment}{$\triangleright$}{}

\def\*#1{\mathbf{#1}}

\usepackage{amsfonts}
\usepackage{mathtools}
\usepackage{amsthm}
\usepackage{bm}

\newcommand{\x}{\mathbf{x}}
\newcommand{\z}{\mathbf{z}}

\usepackage{enumitem}
\usepackage{makecell}
\usepackage{threeparttable}
\usepackage{annotate-equations}

\newcommand{\bbE}{\mathbb{E}}
\newcommand{\norm}[1]{\left\lVert#1\right\rVert_2}
\newcommand{\norml}[1]{\left\lVert#1\right\rVert_{\text{lip}}}
\newcommand{\p}{\mathbb{P}}
\newcommand{\cW}{\mathcal{W}}
\newcommand{\cG}{\mathcal{G}}
\newcommand{\cH}{\mathcal{H}}
\newcommand{\cR}{\mathcal{R}}
\newcommand{\ip}[2]{\langle #1, #2 \rangle}

\usepackage[utf8]{inputenc} 
\usepackage[T1]{fontenc}    
\usepackage{url}            
\usepackage{booktabs}       
\usepackage{amsfonts}       
\usepackage{nicefrac}       
\usepackage{microtype}      
\usepackage{xcolor}         
\usepackage[most]{tcolorbox}
\usepackage{cancel}

\definecolor{mydarkblue}{rgb}{0,0.08,0.45}
\definecolor{mydarkred}{rgb}{0.6,0,0}
\definecolor{myblue}{HTML}{268BD2}
\definecolor{mygreen}{HTML}{658354}
\definecolor{orangeinplot}{HTML}{e29c7a}
\definecolor{purpleinplot}{HTML}{7676a4}
\definecolor{greeninplot}{HTML}{288308}
\definecolor{f_red}{HTML}{911C43}
\definecolor{f_orange}{HTML}{F3C17B}
\definecolor{f_green}{HTML}{C6E4A7}
\definecolor{f_violet}{HTML}{5B509D}
\usepackage[colorlinks,
linkcolor=mydarkred,
citecolor=mydarkblue]{hyperref}

\definecolor{LightCyanBG}{rgb}{0.85,0.92,0.97}
\definecolor{CyanBG}{rgb}{0.71,0.85,0.93}
\definecolor{BlueBG}{rgb}{0,0.46,0.71}

\title{HYPO: Hyperspherical Out-of-Distribution Generalization}

\author{Haoyue Bai$^1$\thanks{Equal contribution. Correspondence to Yifei Ming and Yixuan Li}, Yifei Ming$^1$\footnotemark[1], Julian Katz-Samuels$^2$\footnotemark[2]\thanks{This work is not related to the author's position at Amazon.}, Yixuan Li$^1$ \\
Department of Computer Sciences, University of Wisconsin-Madison$^1\quad$\\
Amazon$^2\quad$\\
\texttt{\{baihaoyue,alvinming,sharonli\}@cs.wisc.edu}, \texttt{jkatzsamuels@gmail.com}\\
}

\iclrfinalcopy 
\begin{document}

\maketitle

\begin{abstract}
Out-of-distribution (OOD) generalization is critical for machine learning models deployed in the real world.  However, achieving this can be fundamentally challenging, as it requires the ability to learn invariant features across different domains or environments. In this paper, we propose a novel framework HYPO (\textbf{HYP}erspherical \textbf{O}OD generalization) that provably learns domain-invariant representations in a hyperspherical space. In particular, our hyperspherical learning algorithm is guided by intra-class variation and inter-class separation principles---ensuring that features from the same class (across different training domains) are closely aligned with their class prototypes, while different class prototypes are maximally separated. We further provide theoretical justifications on how our prototypical learning objective improves the OOD generalization bound. Through extensive experiments on challenging OOD benchmarks, we demonstrate that our approach outperforms competitive baselines and achieves superior performance. 
Code is available at \url{https://github.com/deeplearning-wisc/hypo}.
\end{abstract}

\input{chap_Arxiv/intro}

\input{chap_Arxiv/problem}

\input{chap_Arxiv/motivation}

\input{chap_Arxiv/meth}

\input{chap_Arxiv/expe}

\input{chap_Arxiv/theory}

\input{chap_Arxiv/rela}

\input{chap_Arxiv/conc}

\bibliography{iclr2024_conference}
\bibliographystyle{iclr2024_conference}

\newpage
\appendix

\input{chap_Arxiv/appendix}

\end{document}

%% file: math_commands.tex
\usepackage{mathtools}
\usepackage{amssymb}
\usepackage{latexsym}
\usepackage{dsfont}
\usepackage{mathrsfs}
\usepackage{amssymb}
\usepackage{amsfonts}
\usepackage{bm}
\usepackage{xspace}
\usepackage{amsthm}
\usepackage{multirow}
\newcommand{\R}{\mathbb{R}} 

\newcommand*{\argmax}{\mathop{\mathrm{argmax}}}

\newcommand{\Ecal}{\mathcal{E}}
\newcommand{\Fcal}{\mathcal{F}}

\newcommand{\Ical}{\mathcal{I}}

\newcommand{\Vcal}{\mathcal{V}}

\newcommand{\Xcal}{\mathcal{X}}
\newcommand{\Ycal}{\mathcal{Y}}

\newcommand{\Pbb}{\mathbb{P}}

\newcommand{\Rbb}{\mathbb{R}}

\ifx\BlackBox\undefined
\newcommand{\BlackBox}{\rule{1.5ex}{1.5ex}}  
\fi

\ifx\QED\undefined
\def\QED{~\rule[-1pt]{5pt}{5pt}\par\medskip}
\fi

\ifx\proof\undefined
\newenvironment{proof}{\par\noindent{\em Proof:\ }}{\hfill\BlackBox\\[.0mm]}
\fi

\ifx\theorem\undefined
\newtheorem{theorem}{Theorem}[section]
\fi

\ifx\example\undefined
\newtheorem{example}{Example}[section]
\fi

\ifx\property\undefined
\newtheorem{property}{Property}
\fi

\ifx\lemma\undefined
\newtheorem{lemma}{Lemma}[section]
\fi

\ifx\proposition\undefined
\newtheorem{proposition}{Proposition}[section]
\fi

\ifx\remark\undefined
\newtheorem{remark}{Remark}
\fi

\ifx\corollary\undefined
\newtheorem{corollary}{Corollary}[section]
\fi

\ifx\definition\undefined
\newtheorem{definition}{Definition}[section]
\fi

\ifx\conjecture\undefined
\newtheorem{conjecture}{Conjecture}[section]
\fi

\ifx\axiom\undefined
\newtheorem{axiom}[theorem]{Axiom}
\fi

\ifx\claim\undefined
\newtheorem{claim}[theorem]{Claim}
\fi

\ifx\assumption\undefined
\newtheorem{assumption}{Assumption}
\fi

\ifx\problem\undefined
\newtheorem{problem}{Problem}[section]
\fi

\ifx\fact\undefined
\newtheorem{fact}{Fact}
\fi

\newcommand{\err}{\mathrm{err}}

\newcommand{\benr}{\begin{eqnarray}}
\newcommand{\eenr}{\end{eqnarray}}
\newcommand{\benrr}{\begin{eqnarray*}}
\newcommand{\eenrr}{\end{eqnarray*}}
\newcommand{\ben}{\begin{equation}}
\newcommand{\een}{\end{equation}}
\newcommand{\benn}{\begin{equation*}}
\newcommand{\eenn}{\end{equation*}}

%% file: chap_Arxiv/intro.tex
\section{Introduction}

Deploying machine learning models in real-world settings presents a critical challenge of generalizing under distributional shifts. These shifts are common due to mismatches between the training and test data distributions. For instance, in autonomous driving, a model trained on in-distribution (ID) data collected under sunny weather conditions is expected to perform well in out-of-distribution (OOD) scenarios, such as rain or snow. This underscores the importance of the OOD generalization problem, which involves learning a predictor that can generalize across all possible environments, despite being trained on a finite subset of training environments.

A plethora of OOD generalization algorithms has been developed in recent years~\citep{zhou2022domain}, 
where a central theme is to learn domain-invariant representations---features that are consistent and meaningful across different environments (domains) and can generalize to the unseen test environment. Recently, \citet{ye2021towards} theoretically showed that the OOD generalization error can be bounded in terms of intra-class \emph{variation} and inter-class \emph{separation}. Intra-class variation measures the stability of representations across different environments, while inter-class separation assesses the dispersion of features among different classes. Ideally, features should display low variation and high separation, in order to generalize well to OOD data (formally described in Section~\ref{sec:motivation}). Despite the theoretical analysis,  a research question remains open in the field:
\begin{tcolorbox}[colback=gray!5!white]
 \emph{\textbf{RQ:}} How to design a practical learning algorithm that directly achieves these two properties, and what theoretical guarantees can the algorithm offer?
\end{tcolorbox}

To address the question, this paper presents a learning framework HYPO (\textbf{HYP}erspherical \textbf{O}OD generalization), which \textcolor{black}{provably learns} domain-invariant representations in the hyperspherical space \textcolor{black}{with unit norm} (Section~\ref{sec:method}). Our key idea is to promote low variation (aligning representation across domains for every class) and high separation (separating prototypes across different classes). In particular, the learning objective shapes the embeddings such that samples from the same class (across all training environments) gravitate towards their corresponding class prototype, while different class prototypes are maximally separated. The two losses in our objective function can be viewed as optimizing the key properties of intra-class variation and inter-class separation, respectively. Since samples are encouraged to have a small distance with respect to {their} class prototypes, the resulting embedding geometry can have a small distribution discrepancy across domains and benefits OOD generalization. Geometrically, we show that our loss function can be understood through the lens of maximum likelihood estimation under the classic von Mises-Fisher distribution.

\vspace{-0.2cm}
\paragraph{Empirical contribution.} 
Empirically, we demonstrate strong OOD generalization performance by extensively evaluating HYPO on common benchmarks (Section~\ref{sec:experiment}). On the CIFAR-10 (ID) vs. CIFAR-10-Corruption (OOD) task, HYPO substantially improves the OOD generalization accuracy on challenging cases such as Gaussian noise, from $78.09\%$ to $85.21\%$.  Furthermore, we establish superior performance on popular domain generalization benchmarks, including PACS, Office-Home, VLCS, etc. For example, we achieve 88.0\%
accuracy on PACS which outperforms the best loss-based method by 1.1\%. 
This improvement is non-trivial using standard stochastic gradient descent optimization. When coupling our loss with specialized optimization SWAD~\citep{cha2021swad}, the accuracy is further increased to 89\%. We provide visualization and quantitative analysis to verify that features learned by HYPO indeed achieve low intra-class variation and high inter-class separation.

\vspace{-0.2cm}
\paragraph{Theoretical insight.} 
We provide theoretical justification for how HYPO can guarantee improved OOD generalization, supporting our empirical findings. Our theory complements~\cite{ye2021towards}, which does not provide a loss for optimizing the intra-class variation or inter-class separation. Thus, \emph{a key contribution of this paper is to provide a crucial link between provable understanding and a practical algorithm for OOD generalization in the hypersphere.} In particular, our Theorem~\ref{thm:main_theorem} shows that when the model is trained with our loss function, we can upper bound intra-class variation, a key quantity to bound OOD generalization error. For a learnable OOD generalization task, the upper bound on generalization error is determined by the variation estimate on the training environments, which is effectively reduced by our loss function under sufficient sample size and expressiveness of the neural network.

%% file: chap_Arxiv/problem.tex
\section{Problem Setup}

We consider a multi-class classification task that involves a pair of random variables $(X,Y)$ over instances 
$\*x \in \mathcal{X} \subset \mathbb{R}^d$ and corresponding labels $y \in \mathcal{Y} := \{1,2,\cdots, C\}$. The joint distribution of $X$ and $Y$ is unknown and represented by $\mathbb{P}_{XY}$. The goal is to learn a predictor function, $f: \mathcal{X} \rightarrow \mathbb{R}^C$, that can accurately predict the label $y$ for an input $\*x$, where  $(\*x,y)\sim \mathbb{P}_{XY}$. 

Unlike in standard supervised learning tasks, the out-of-distribution (OOD) generalization problem is challenged by the fact that one cannot sample directly from $\mathbb{P}_{XY}$. Instead,  
we can only sample $(X,Y)$ under limited environmental conditions, each of which corrupts or varies the data differently. For example, in autonomous driving, these environmental conditions may represent different weathering conditions such as snow, rain, etc. We formalize this notion of environmental variations with a set of \emph{environments} or domains $\mathcal{E}_\text{all}$. Sample pairs $(X^e, Y^e)$ are randomly drawn from environment $e$. 
In practice, we may only have samples from a finite subset of \emph{available environments} $\mathcal{E}_\textnormal{avail} \subset \mathcal{E}_\text{all}$. Given $\mathcal{E}_\text{avail}$, the goal is to learn a predictor $f$ that can generalize across all possible environments. The problem is stated formally below.

\begin{definition}[{OOD Generalization}]\label{def_generalization}
Let $\mathcal{E}_\textnormal{avail} \subset \mathcal{E}_\textnormal{all}$ be a set of training environments, and assume that for each environment $e \in \mathcal{E}_\textnormal{avail}$, we have a dataset  $\mathcal{D}^e = \{(\*x_j^e, y_j^e)\}_{j=1}^{n_e}$, sampled {i.i.d.} from an unknown distribution $\mathbb{P}_{XY}^e$. The goal of OOD generalization is to find a classifier $f^*$, using the data
from the datasets $\mathcal{D}^e$,
that minimizes the worst-case risk over the entire family of environments $\mathcal{E}_\textnormal{all}$:
\begin{equation}
    \min_{f\in \mathcal{F}} \max_{ e \in \mathcal{E}_\textnormal{all}} \mathbb{E}_{\mathbb{P}_{XY}^e} \ell(f(X^e),Y^e),
\end{equation}
where $\mathcal{F}$ is hypothesis space and $l(\cdot, \cdot)$ is the loss function. 
\end{definition}
The problem is challenging since we do not have access to data from domains outside $\mathcal{E}_\text{avail}$. In particular, the task is commonly referred to as {multi-source domain generalization} when $|\mathcal{E}_\text{avail}| > 1$.

%% file: chap_Arxiv/motivation.tex
\section{Motivation of Algorithm Design}
\label{sec:motivation}

Our work is motivated by the theoretical findings in~\cite{ye2021towards}, which shows that the OOD generalization performance can be bounded in terms of intra-class  \emph{variation} and inter-class \emph{separation} with respect to various environments.
The formal definitions are given as follows.  
\begin{definition}[Intra-class variation]\label{def_invariance}
The variation of feature $\phi$ across a domain set $\Ecal$ is
\begin{align}
\label{eq:variation}
\Vcal(\phi, \Ecal) = \max_{y\in\Ycal} \sup_{e, e'\in\Ecal} \rho \big(\Pbb(\phi^e|y), \Pbb(\phi^{e'}|y) \big),
\end{align}
where $\rho(\mathbb P,\mathbb Q)$ is a symmetric distance (e.g., Wasserstein distance, total variation, Hellinger distance) between two distributions, and $\Pbb(\phi^e|y)$ denotes the class-conditional distribution for features of samples in environment $e$.
\end{definition}

\begin{definition}[Inter-class separation\footnote{Referred to as ``{Informativeness}'' in \citet{ye2021towards}.}]\label{def_informativeness}
The separation of feature $\phi$ across domain set $\Ecal$ is
\begin{align}
\label{eq:separation}
\Ical_\rho(\phi, \Ecal) = 
\frac{1}{C(C-1)}\sum_{\substack{y\neq y'\\ y,y'\in \Ycal}}
\min_{e\in\Ecal} \rho \big( \Pbb(\phi^e|y) , \Pbb(\phi^e|y')\big).
\end{align}
\end{definition} 

The intra-class variation $\Vcal(\phi, \Ecal)$ measures the stability of feature $\phi$ over the domains in $\Ecal$ and the inter-class separation $\Ical(\phi, \Ecal)$ captures the ability of $\phi$ to distinguish different labels. Ideally, features should display high separation and low variation. {\color{black}

}

\begin{definition}
    The OOD generalization error of classifier $f$ is defined as follows: $$\text{err}(f)  = \max_{ e \in \mathcal{E}_\text{all}} \mathbb{E}_{\mathbb{P}_{XY}^e} \ell(f(X^e),Y^e) - \max_{ e \in \mathcal{E}_\text{avail}} \mathbb{E}_{\mathbb{P}_{XY}^e} \ell(f(X^e),Y^e)$$ which is bounded by the variation estimate on $\mathcal{E}_{\text {avail }}$ with the following theorem.
\end{definition}

{\color{black}
\begin{theorem}[OOD error upper bound, informal~\citep{ye2021towards}]\label{general ood bound} Suppose the loss function $\ell(\cdot, \cdot)$ is bounded by $[0, {B}]$. {\color{black} For a learnable OOD generalization problem with sufficient inter-class separation}, the OOD generalization error $\text{err}(f)$  can be upper bounded by
\begin{equation}\label{main}
\err(f) \leq O\Big(\big(\Vcal^{\textnormal{sup}}(h,\Eava)\big)^{\frac{\alpha^2}{(\alpha+d)^2}}\Big),
\end{equation}
for some $\alpha>0 $, and $\mathcal{V}^{\textnormal {\textnormal{sup}}}\left(h, \mathcal{E}_{\textnormal {\textnormal{avail} }}\right)  \triangleq \sup _{\beta \in \mathcal{S}^{d-1}} \mathcal{V}\left(\beta^{\top} h, \mathcal{E}_{\textnormal {avail }}\right)$ is the inter-class variation,  $h(\cdot) \in \mathbb{R}^d$ is the feature vector,  and $\beta$ is a vector in unit hypersphere $\mathcal{S}^{d-1}=\left\{\beta \in \mathbb{R}^d:\|\beta\|_2=1\right\}$, and f is a classifier based on normalized feature $h$. 
\end{theorem}
}

\textbf{Remarks.} 
The Theorem above suggests that both low intra-class variation and high inter-class separation are desirable properties for theoretically grounded OOD generalization. Note that in the \textbf{full formal} Theorem (see Appendix~\ref{sec:proof}), maintaining the inter-class separation is necessary for the learnability of the OOD generalization problem (Def.~\ref{def_learn}). In other words, when the learned embeddings exhibit high inter-class separation, the problem becomes learnable. In this context, bounding intra-class variation becomes crucial for reducing the OOD generalization error.

Despite the
theoretical underpinnings, it remains unknown to the field how to design a practical learning algorithm that directly achieves these two properties,
and what theoretical guarantees can the algorithm offer. This motivates our work.

\begin{tcolorbox}[colback=gray!5!white]
To reduce the OOD generalization error, our key motivation is to design a {hyperspherical} learning algorithm that directly promotes low variation (aligning representation across domains for every class) and high separation (separating prototypes across different classes). 
\end{tcolorbox}

%% file: chap_Arxiv/meth.tex
\vspace{-0.2cm}
\section{Method}
\label{sec:method}

{\color{black}Following the motivation in Section~\ref{sec:motivation}, we now introduce the details of the learning algorithm HYPO (\textbf{HYP}erspherical \textbf{O}OD generalization), which is designed to promote domain invariant representations in the hyperspherical space.  
}
The key idea is to shape the hyperspherical embedding
space so that samples from the same class (across all training environments $\mathcal{E}_\text{avail}$) are closely aligned with the corresponding class prototype. Since all points are encouraged to have a small distance with respect to the class prototypes, the resulting embedding geometry can have a small distribution discrepancy across domains and hence benefits OOD generalization. {\color{black}In what follows, we first introduce \textcolor{black}{the learning objective} (Section~\ref{sec:hyper}), and then we discuss the geometrical interpretation of the loss and embedding (\textcolor{black}{Section~\ref{sec:geo}}). {We will} provide theoretical justification for HYPO in Section~\ref{sec:theory}, which leads to a provably smaller intra-class variation, a key quantity to bound OOD generalization error.}

\subsection{Hyperspherical Learning for OOD Generalization} 
\label{sec:hyper}

\paragraph{Loss function.} {The learning algorithm is motivated to directly optimize the two criteria: intra-class variation and inter-class separation.}
At a high level, HYPO aims to learn embeddings 
for each sample in the training environments by maintaining a class prototype vector $\boldsymbol{\mu}_c \in \mathbb{R}^d$ for each class $c \in \{1,2,...,C\}$. To optimize for low variation, the loss encourages the feature embedding of a sample to be close to its class prototype.  To optimize for high separation, the loss encourages different class prototypes to be far apart from each other.

 Specifically, we consider a deep neural network $h: \Xcal \mapsto \Rbb^d$ that maps an input $\tilde{\x} \in \Xcal$ to a feature embedding $\tilde{\*z} := h(\tilde{\x})$. The loss operates on the {normalized} feature embedding $\*z := \tilde{\*z} / \lVert \tilde{\*z}\rVert_2$. The normalized embeddings are also referred to as \emph{hyperspherical embeddings}, since they are on a unit hypersphere, denoted as $S^{d-1} := \{\z \in \mathbb{R}^{d} ~|~\lVert \z \rVert_2 = 1\}$.
The loss is formalized as follows: 
\vspace{-0.01cm}
\begin{equation}
\label{eq:loss}
     \mathcal{L} =   \underbrace{- \frac{1}{N} \sum_{e \in \mathcal{E}_\text{avail}} \sum_{i=1}^{|\mathcal{D}^e|} \log \frac{\exp \left({\*z^{e}_{i}}^\top  {\boldsymbol{\mu}}_{c(i)} / \tau\right)}{\sum_{j=1}^{C} \exp \left({\*z^{e}_{i}}^\top  {\boldsymbol{\mu}}_{j} / \tau\right)}}_\text{$\mathcal{L}_\text{var} $:~~$\downarrow$ \textbf{variation}} \notag 
     + \underbrace{\frac{1}{C} \sum_{i=1}^C \log {1\over C-1}\sum_{\substack{j\neq i, j \in \Ycal}}   \exp{\left({\boldsymbol{\mu}}_i^\top {\boldsymbol{\mu}}_j/\tau\right)}}_\text{$\uparrow$ \textbf{separation}},
\end{equation}

where $N$ is the number of samples, $\tau$ is the temperature, $\*z$ is the normalized feature embedding, {and $\boldsymbol{\mu}_c$ is the prototype embedding for class $c$}. While hyperspherical learning algorithms have been studied in other context~\citep{mettes2019hyperspherical, 2020supcon, ming2023cider}, 
\emph{none of the prior works explored its provable connection to  domain generalization, which is our distinct contribution.}  
\emph{We will theoretically show in Section~\ref{sec:theory} that minimizing our loss function effectively reduces intra-class variation, a key quantity to bound OOD generalization error.}

The training objective in Equation~\ref{eq:loss} can be efficiently optimized end-to-end. During training, an important step is to estimate the class prototype $\boldsymbol{\mu}_{c}$ for each class $c\in\{1,2,...,C\}$.  
The class-conditional prototypes can be updated in an exponential-moving-average manner (EMA)~\citep{li2020mopro}:
 \begin{equation}
 \label{eq:update}
     {\boldsymbol{\mu}}_c := \text{Normalize}( \alpha{\boldsymbol{\mu}}_c + (1-\alpha)\*z), 
     \; \forall c\in\{1, 2, \ldots, C\}
 \end{equation}
 where the prototype $\boldsymbol{\mu}_{c}$ for class $c$ is updated during training as the moving average of all embeddings with label $c$, and $\*z$ denotes the normalized embedding of samples of class $c$. {An end-to-end pseudo algorithm is summarized in Appendix~\ref{alg}}.

\paragraph{Class prediction.} In testing, classification is conducted by identifying the closest class prototype: { $\hat{y}= \argmax_{c \in [C]} f_c(\*x)$, where $f_c(\*x) =\*z^\top\boldsymbol{\mu}_c$} and $\*z = {h(\*x)\over \|h(\*x)\|_2 } $ is the normalized feature embedding.

\subsection{Geometrical Interpretation of Loss and Embedding}
\label{sec:geo}

Geometrically, the loss function above can be interpreted as learning embeddings located on the surface of a unit hypersphere. The hyperspherical embeddings can be modeled by the von Mises-Fisher (vMF) distribution, a well-known distribution in directional statistics~\citep{jupp2009directional}. For a unit vector $\z \in \Rbb^d$ in class $c$, the probability density function is defined as
\begin{equation} \label{eq:conditional}
    p(\z \mid y=c) = Z_d(\kappa)\exp(\kappa\boldsymbol{\mu}_{c}^\top \z),
\end{equation}

where $\boldsymbol{\mu}_{c} \in \mathbb{R}^d$ denotes the mean direction of the class $c$, $\kappa \geq 0$ denotes the concentration of the distribution around $\boldsymbol{\mu}_{c}$, and $Z_d(\kappa)$ denotes the normalization factor. A larger $\kappa$ indicates a higher concentration around the class center.  In the extreme case of $\kappa=0$, the samples are distributed
uniformly on the hypersphere.

\begin{wrapfigure}[11]{r}{0.42\textwidth}
    \centering    \includegraphics[width=0.38\textwidth]{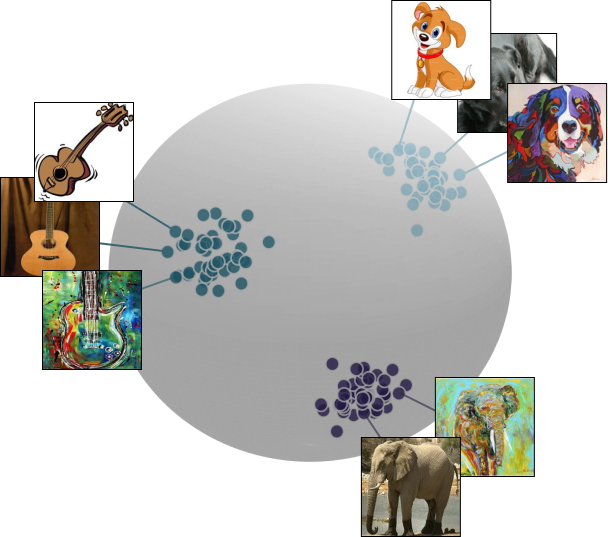}
    \vspace{-0.1cm}
    \caption{\small Illustration of hyperspherical embeddings. Images are from PACS~\citep{li2017deeper}.}
    \label{fig:hyper}
\end{wrapfigure}

Under this probabilistic model, an embedding $\z$ is assigned to the class $c$ with the following probability
\begin{align} \label{eq:probability vMF}
    p(y=c \mid \z; \{\kappa, \boldsymbol{\mu}_{j}\}_{j=1}^C)
        &= \frac{Z_d(\kappa) \exp(\kappa \boldsymbol{\mu}_{c}^\top \mathbf{z})}{\sum_{j=1}^C Z_d(\kappa) \exp(\kappa \boldsymbol{\mu}_{j}^\top \mathbf{z})} \nonumber  \\
        &= \frac{\exp(\boldsymbol{\mu}_{c}^\top \mathbf{z} / \tau)}{\sum_{j=1}^C \exp(\boldsymbol{\mu}_{j}^\top \mathbf{z} / \tau)},
\end{align}
where $\tau = 1 / \kappa$ denotes a temperature parameter. 

\noindent\textbf{Maximum likelihood view.} {Notably, minimizing the first term in our loss ({\emph{cf.}} Eq.~\ref{eq:loss}) is equivalent to performing maximum likelihood estimation under the vMF distribution}:
 \begin{equation*} 
    \text{argmax}_\theta \prod_{i=1}^N p(y_i \mid \x_i;   \{\kappa, \boldsymbol{\mu}_{j}\}_{j=1}^C), \text{where}~(\x_i, y_i) \in {\bigcup\limits_{e \in \mathcal{E}_\text{train}} \mathcal{D}^e}
\end{equation*}
where $i$ is the index of sample, $j$ is the index of the class, and $N$ is the size of the training set. In effect, this loss encourages each ID sample to have a high probability assigned to the correct class in the mixtures of the vMF distributions.

%% file: chap_Arxiv/expe.tex
\vspace{-0.3cm}
\section{Experiments}
\label{sec:experiment}

In this section, we show that HYPO achieves strong OOD generalization performance in practice, establishing competitive performance on several benchmarks. In what follows, we describe the experimental setup in Section~\ref{exp:setup}, followed by main results and analysis in Section~\ref{exp:results}.

\subsection{Experimental Setup}\label{exp:setup}

\paragraph{Datasets.} Following the common benchmarks in literature, we use CIFAR-10~\citep{krizhevsky2009learning} as the in-distribution data. We use CIFAR-10-C~\citep{hendrycks2018benchmarking} as OOD data, with 19 different common corruption applied to CIFAR-10. 
In addition to CIFAR-10, we conduct experiments on popular benchmarks including PACS~\citep{li2017deeper}, Office-Home~\citep{gulrajani2020search}, and VLCS~\citep{gulrajani2020search} to validate the generalization performance. PACS contains 4 domains/environments (\texttt{photo, art painting, cartoon, sketch}) with 7 classes (\texttt{dog, elephant, giraffe, guitar, horse, house, person}). Office-Home comprises four different domains: art, clipart, product, and real. 
Results on additional OOD datasets Terra Incognita~\citep{gulrajani2020search}, and ImageNet can be found in Appendix~\ref{sec:otherood} and Appendix~\ref{sec:imagenet}. 

\paragraph{Evaluation metrics.} We report the following two metrics: \textbf{(1)} ID classification accuracy (ID Acc.) for ID generalization, and \textbf{(2)} OOD classification accuracy (OOD Acc.) for OOD generalization.

\begin{table*}[t]
\centering
\scalebox{0.8}{\begin{tabular}{lcccc}
\toprule
\textbf{Algorithm}  & \textbf{PACS} & \textbf{Office-Home} & \textbf{VLCS} & \textbf{Average Acc. (\%)}\\
\midrule
\textbf{ERM}~\citep{vapnik1999overview} & 85.5 & 67.6 & 77.5 & 76.7 \\
\textbf{CORAL}~\citep{sun2016deep}   & 86.2 & 68.7 & 78.8 & 77.9 \\
\textbf{DANN}~\citep{ganin2016domain}   & 83.7 & 65.9 & 78.6 & 76.1  \\
\textbf{MLDG}~\citep{li2018learning} & 84.9 & 66.8 & 77.2 & 76.3 \\
\textbf{CDANN}~\citep{li2018deep}  & 82.6 & 65.7 & 77.5 & 75.3 \\
\textbf{MMD}~\citep{li2018domain}   & 84.7 & 66.4 & 77.5 & 76.2 \\
\textbf{IRM}~\citep{arjovsky2019invariant}  & 83.5 & 64.3 & 78.6 & 75.5\\
\textbf{GroupDRO}~\citep{sagawa2020distributionally}   & 84.4 & 66.0 & 76.7 & 75.7 \\
\textbf{I-Mixup}~\citep{wang2020heterogeneous,xu2020adversarial,yan2020improve}  & 84.6 & 68.1 & 77.4 & 76.7 \\
\textbf{RSC}~\citep{huang2020self}  & 85.2 & 65.5 & 77.1 & 75.9 \\
\textbf{ARM}~\citep{zhang2020adaptive}  & 85.1 & 64.8 & 77.6 & 75.8 \\
\textbf{MTL}~\citep{blanchard2021domain}  & 84.6 & 66.4 & 77.2 & 76.1 \\
\textbf{VREx}~\citep{krueger2021out}  & 84.9 & 66.4 & 78.3 & 76.5 \\
\textbf{Mixstyle}~\citep{zhou2021domain}   & 85.2 & 60.4 & 77.9 & 74.5 \\
\textbf{SelfReg}~\citep{kim2021selfreg}  & 85.6 & 67.9 & 77.8 & 77.1 \\
\textbf{SagNet}~\citep{nam2021reducing} & 86.3 & 68.1 & 77.8 & 77.4  \\
\textbf{GVRT}~\citep{min2022grounding} & 85.1 & 70.1 & 79.0 & 78.1 \\
\textbf{VNE}~\citep{kim2023vne} & 86.9 & 65.9 & 78.1 & 77.0 \\
\textbf{HYPO (Ours)}  &  88.0$_{\pm 0.4}$  & 71.7$_{\pm 0.3}$ &  78.2$_{\pm 0.4}$ & \textbf{79.3} \\
\bottomrule
\end{tabular}}         
\vspace{-0.2cm}
\caption[]{\small {\color{black} Comparison with domain generalization methods on the PACS, Office-Home, and VLCS. All methods are trained on ResNet-50. The model selection is based on a training domain validation set. {To isolate the effect of loss functions, all methods are optimized using standard SGD}. 
We report the average and std of our method. $\pm x$ denotes the rounded standard error. }
}
\label{tab:all}
\end{table*}

\vspace{-0.1cm}
\paragraph{Experimental details.} In our main experiments, we use ResNet-18 for CIFAR-10 and ResNet-50 for PACS, Office-Home, and VLCS. 
For these datasets, we use stochastic gradient descent with momentum 0.9, and weight decay $10^{-4}$.
For {CIFAR-10}, we train the model from scratch for 500 epochs using an initial learning rate of 0.5 and cosine scheduling, with a batch size of 512. Following common practice for contrastive losses~\citep{2020simclr,2020supcon,yao2022pcl}, we use an MLP projection head with one hidden layer to obtain features. The embedding (output) dimension is 128 for the projection head. We set the default temperature $\tau$ as 0.1 and the prototype update factor $\alpha$ as 0.95. 
For PACS, Office-Home, and VLCS, we follow the common practice and initialize the network using ImageNet pre-trained weights. We fine-tune the network for 50 epochs.
The embedding dimension is 512 for the projection head. We adopt the leave-one-domain-out evaluation protocol and use the training domain validation set for model selection~\citep{gulrajani2020search}, where the validation set is pooled from all training domains. Details on other hyperparameters are in Appendix~\ref{sec:expdetails}.

\vspace{-0.2cm}
\subsection{Main Results and Analysis}\label{exp:results}

\vspace{-0.1cm}
\textbf{HYPO excels on common corruption benchmarks.}
As shown in Figure~\ref{fig:c10_corr}, HYPO achieves {consistent improvement over the ERM baseline (trained with cross-entropy loss), on a variety of common corruptions}. Our evaluation includes different corruptions including Gaussian noise, Snow, JPEG compression, Shot noise, Zoom blur, etc.
The model is trained on CIFAR-10, without seeing any type of corruption data. In particular, our method  brings significant improvement for challenging cases such as Gaussian noise, enhancing OOD accuracy from $78.09\%$ to $85.21\%$ ($+\textbf{7.12}\%$). Complete results on all 19 different corruption types are in Appendix~\ref{sec:c10_shift}.

\begin{figure*}[h]
\begin{center}
\includegraphics[width=0.9\linewidth]{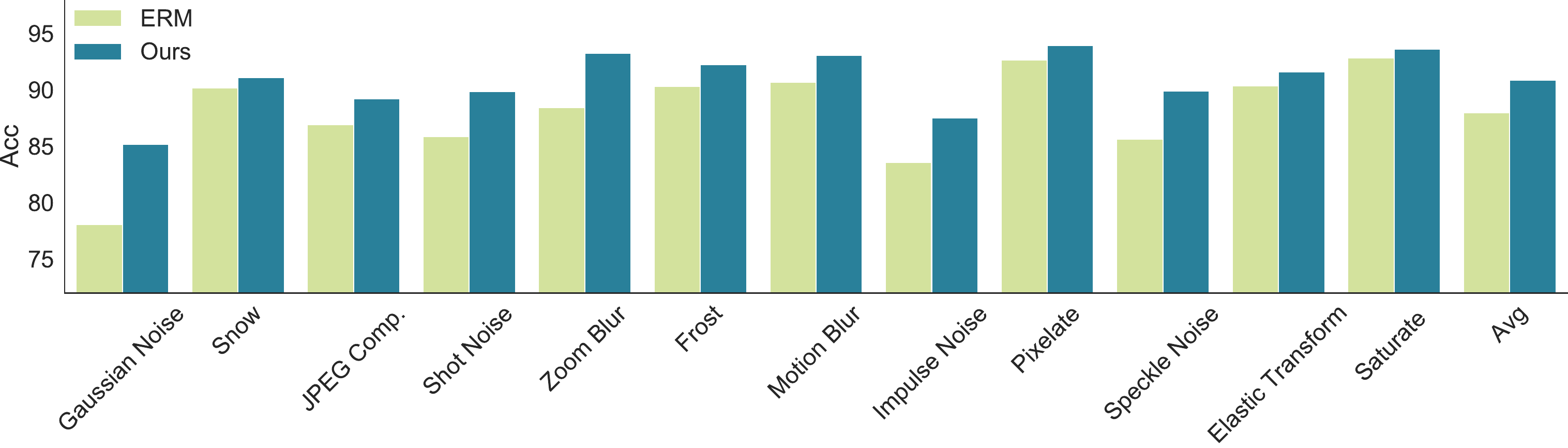}
\end{center}
\vspace{-0.2in}
\caption{\small Our method HYPO significantly improves the OOD generalization performance compared to ERM  on various OOD datasets w.r.t. CIFAR-10 (ID). Full results can be seen in Appendix~\ref{sec:c10_shift}.}
\label{fig:c10_corr}
\end{figure*}

\noindent\textbf{HYPO establishes competitive performance on popular benchmarks.}
Our method delivers superior results in the popular domain generalization tasks, as shown in Table~\ref{tab:all}. HYPO outperforms an extensive collection of common OOD generalization baselines on popular domain generalization datasets, including PACS, Office-Home, VLCS.
For instance, on PACS, HYPO improves the best loss-based method by $\textbf{1.1}\%$. Notably, this enhancement is non-trivial since we are not relying on specialized optimization algorithms such as SWAD~\citep{cha2021swad}. Later in our ablation, we show that coupling HYPO with SWAD can further boost the OOD generalization performance, establishing superior performance on this challenging task.

With multiple training domains, we observe that it is desirable to emphasize hard negative pairs when optimizing the inter-class separation. As depicted in {Figure}~\ref{fig:hard_neg}, the embeddings of negative pairs from the same domain but different classes (such as dog and elephant in art painting) can be quite close on the hypersphere. Therefore, it is more informative to separate such hard negative pairs. This can be enforced by a simple modification to the denominator of our variation loss (Eq.~\ref{eq:hard_neg} in Appendix~\ref{sec:expdetails}), {which we adopt for multi-source domain generalization tasks}.

\begin{wrapfigure}[16]{r}{0.35\textwidth}
    \vspace{-0.1in}
    \centering
    \includegraphics[width=0.35\textwidth]{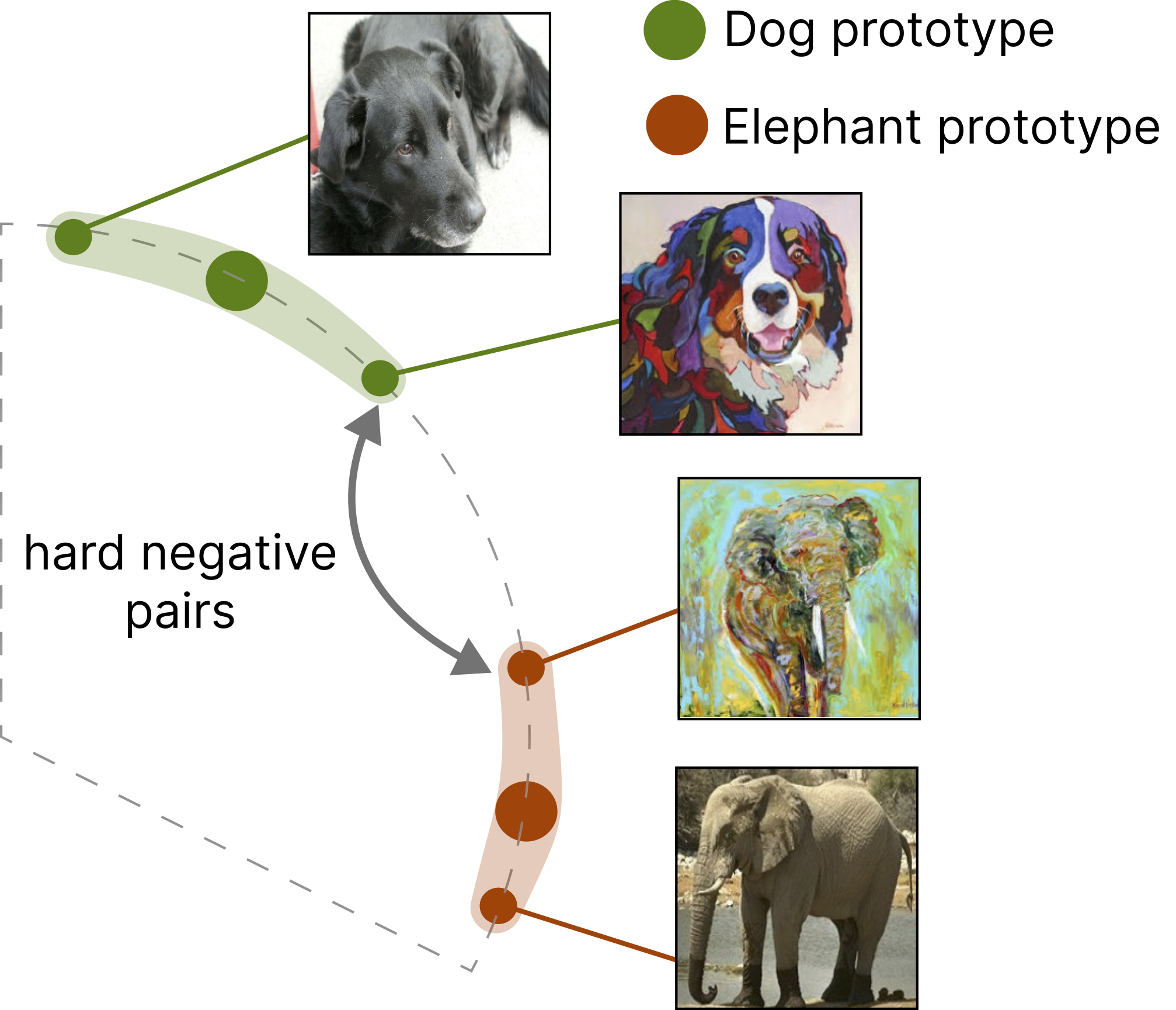}
    \caption{\small Illustration of hard negative pairs which share the same domain (art painting) but have different class labels.}
    \label{fig:hard_neg}
\end{wrapfigure}

\paragraph{Relations to PCL.} PCL~\citep{yao2022pcl} adapts a proxy-based contrastive learning framework for domain generalization. We highlight several notable distinctions from ours: 
\textbf{(1)} While PCL offers no theoretical insights, HYPO is guided by theory. We provide a formal theoretical justification that our method reduces intra-class variation which is essential to bounding OOD generalization error (see Section~\ref{sec:theory}); 
 \textbf{(2)} \textcolor{black}{Our loss function formulation is different} and can be rigorously interpreted as shaping vMF distributions of hyperspherical embeddings (see Section~\ref{sec:geo}), whereas PCL can not; \textbf{(3)} Unlike PCL (86.3\% w/o SWAD), HYPO is able to achieve competitive performance (88.0\%) without heavy reliance on special optimization SWAD~\citep{cha2021swad}, a dense and overfit-aware stochastic weight sampling~\citep{izmailov2018averaging} strategy for OOD generalization. As shown in Table~\ref{tab:pcl-swad}, we also conduct experiments in conjunction
with SWAD. Compared to PCL, HYPO achieves superior performance with \textbf{89}\% accuracy, which further demonstrates its advantage.

\begin{table*}[t]
\centering
\scalebox{0.82}{\begin{tabular}{lccccc}
\toprule
\textbf{Algorithm}  & \textbf{Art painting} & \textbf{Cartoon} & \textbf{Photo} & \textbf{Sketch} & \textbf{Average Acc. (\%)}\\
\midrule
\textbf{PCL w/ SGD}~\citep{yao2022pcl} & 88.0 & 78.8 & 98.1 & 80.3 & 86.3 \\
\textbf{HYPO w/ SGD (Ours)} &    87.2  &  82.3   &   98.0  & 84.5   &  \textbf{88.0}  \\
\midrule
\textbf{PCL w/ SWAD}~\citep{yao2022pcl} & 90.2 & 83.9 & 98.1 & 82.6 & 88.7 \\
\textbf{HYPO w/ SWAD (Ours)} & 90.5 & 84.6 & 97.7 & 83.2 & \textbf{89.0} \\
\bottomrule
\end{tabular}}         
\vspace{-0.15cm}
\caption[]{\small \color{black} Results for comparing PCL and HYPO with SGD-based and SWAD-based optimizations on the PACS benchmark. (*The performance reported in the original PCL paper Table 3 is implicitly based on SWAD).
}
\vspace{-0.3cm}
\label{tab:pcl-swad}
\end{table*}

\paragraph{Visualization of embedding.}
Figure~\ref{fig:umap} shows the UMAP~\citep{mcinnes2018umap-software} visualization of feature embeddings for ERM (left) vs. HYPO (right). The embeddings are extracted from models trained on PACS. The \textcolor{f_red}{red}, \textcolor{f_orange}{orange}, and \textcolor{f_green}{green} points are from the in-distribution, corresponding to art painting (A), photo (P), and sketch (S) domains. The \textcolor{f_violet}{violet} points are from the unseen OOD domain cartoon (C). There are two salient observations: \textbf{(1)} for any given class, the embeddings across domains $\mathcal{E}_\text{all}$ become significantly more aligned (and invariant) using our method compared to the ERM baseline. This directly verifies the low variation ({\emph{cf.}} Equation~\ref{eq:variation}) of our learned embedding. \textbf{(2)} The embeddings are well separated across different classes, and distributed more uniformly in the space than ERM, which verifies the high inter-class separation ({\emph{cf.}} Equation~\ref{eq:separation}) of our method. Overall, our observations well support the efficacy of HYPO.

\begin{figure*}[h]
\centering
     \begin{subfigure}{0.36\textwidth}
         \centering
         \includegraphics[width=\textwidth]{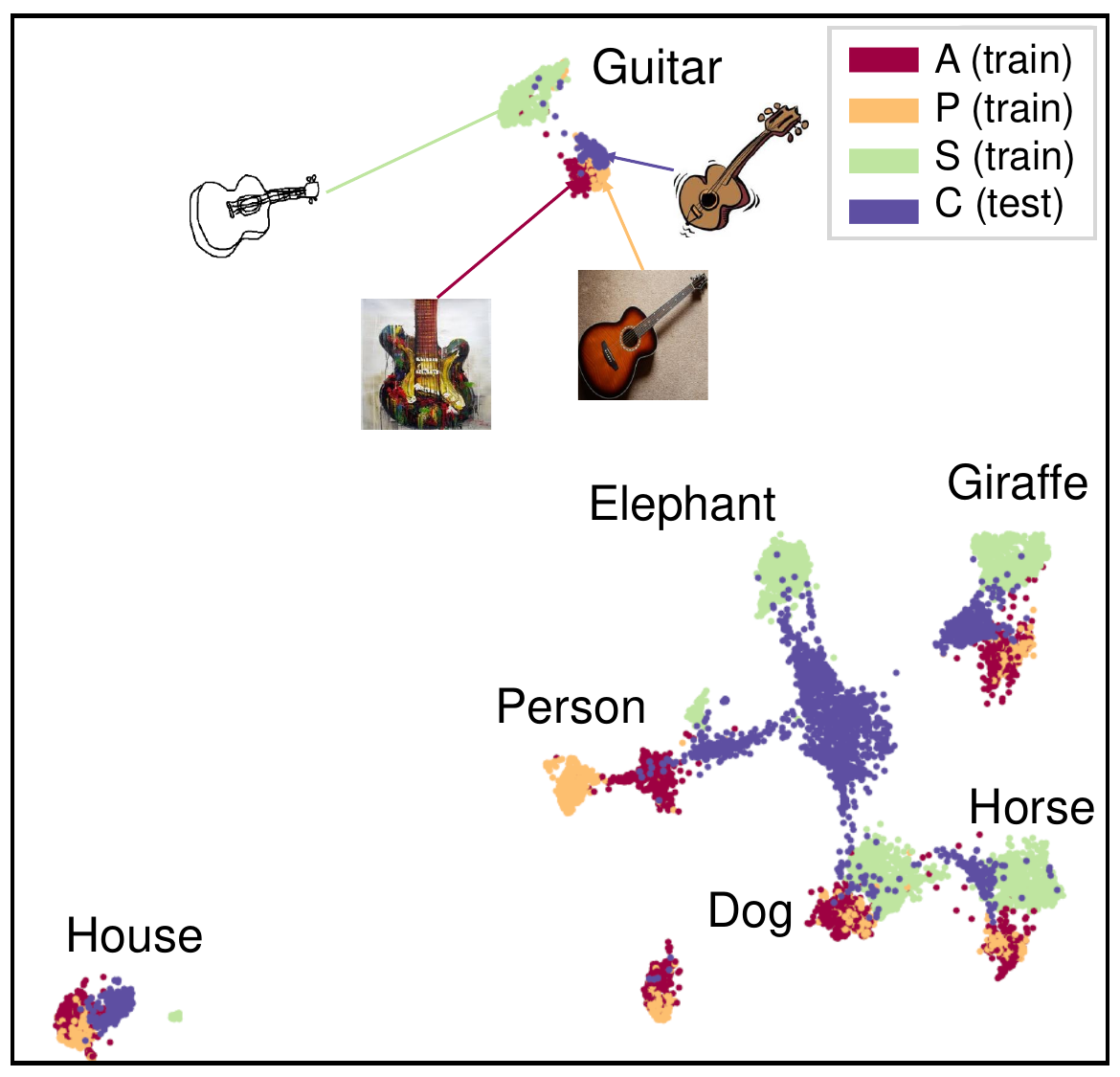}
         \caption{ERM (\textbf{high variation})}
         \label{fig:erm}
     \end{subfigure}
     \hspace{9mm}
     \begin{subfigure}{0.36\textwidth}
         \centering
         \includegraphics[width=\textwidth]{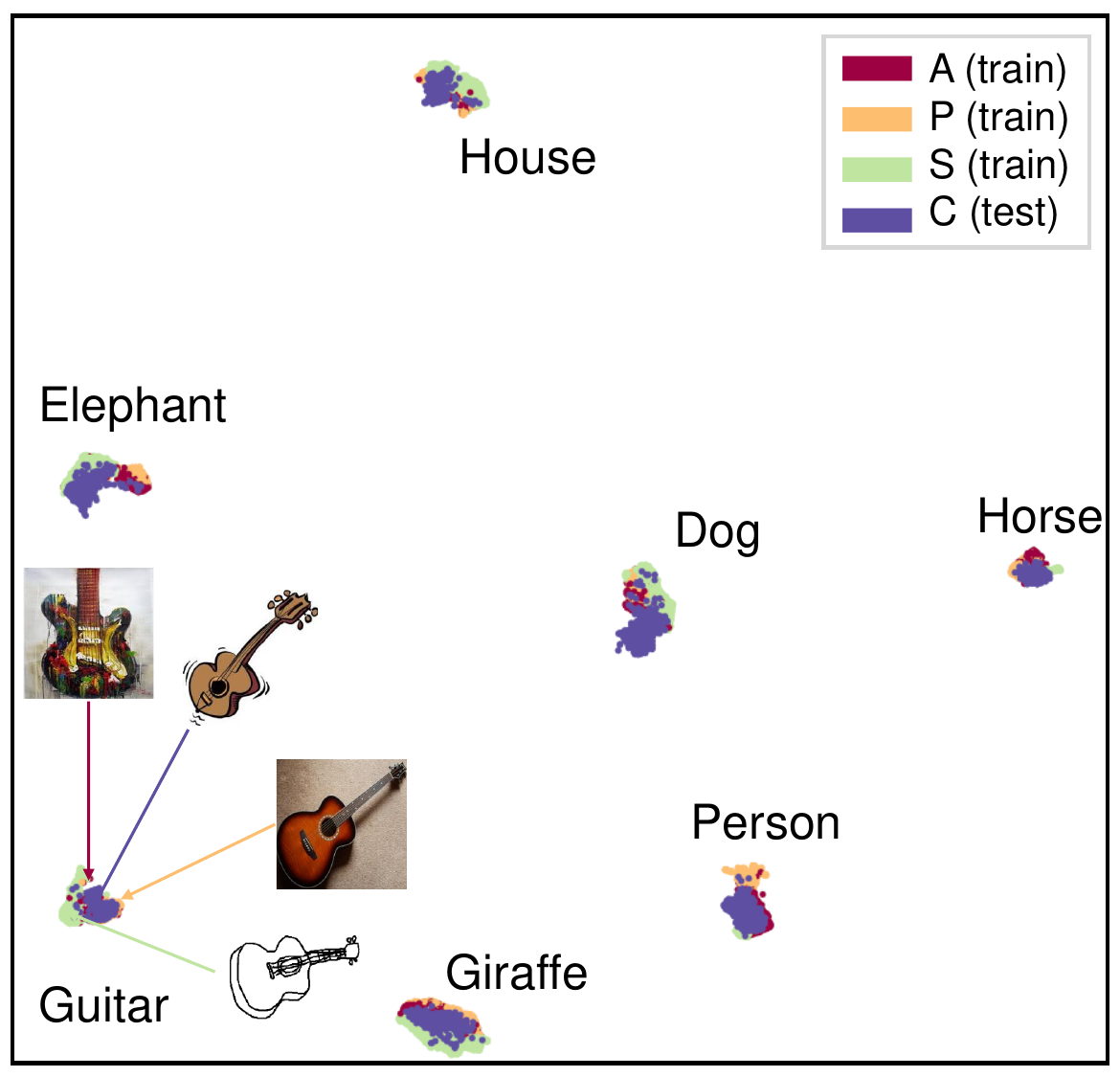}
         \caption{HYPO (\textbf{low variation})}
         \label{fig:ours}
     \end{subfigure}
\vspace{-0.2cm}
\caption[]{\small UMAP~\citep{mcinnes2018umap-software} visualization of the features when the model is trained with \texttt{CE} vs. \texttt{HYPO} for PACS. The \textcolor{f_red}{red}, \textcolor{f_orange}{orange}, and \textcolor{f_green}{green} points are from the in-distribution, which denote art painting (A), photo (P), and sketch (S). The \textcolor{f_violet}{violet} points are from the unseen OOD domain cartoon (C).}
\vspace{-0.1cm}
\label{fig:umap}
\end{figure*}

\vspace{-0.2cm}
\paragraph{Quantitative verification of intra-class variation.}
We provide empirical verification on intra-class variation in Figure~\ref{fig:dist_comp}, where the model is trained on PACS. We measure the intra-class \emph{variation} with {Sinkhorn divergence} (entropy regularized Wasserstein distance). The horizontal axis (0)-(6) denotes different classes, and the vertical axis denotes different pairs of training domains (`P', `A', `S'). Darker color indicates lower Sinkhorn divergence. We can see that our method results in significantly lower intra-class variation compared to ERM, which aligns with our theoretical insights in Section~\ref{sec:theory}.

\begin{figure*}[t]
\vspace{-0.1cm}
\centering
\includegraphics[width=0.89\textwidth]{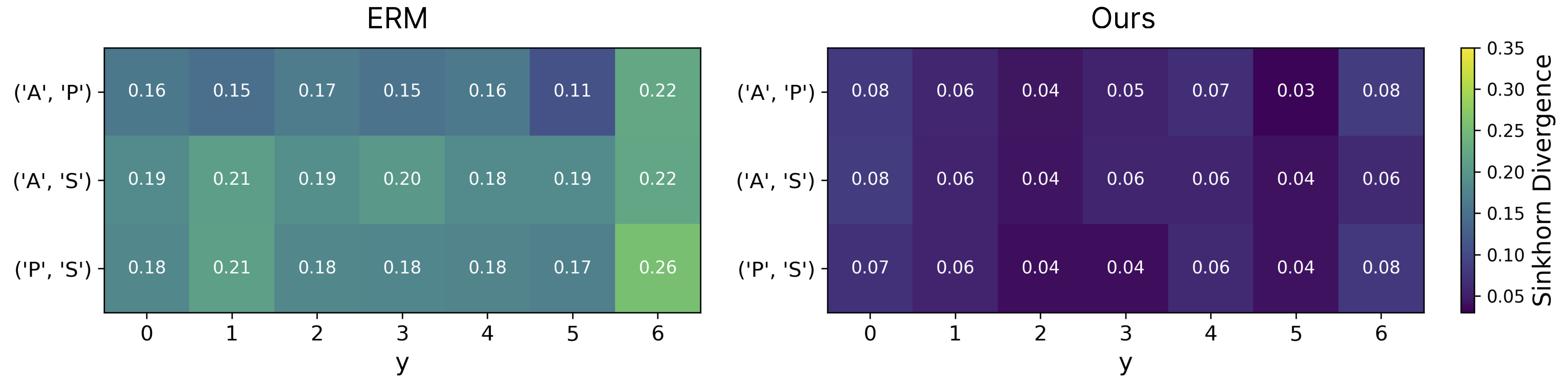}
\vspace{-0.3cm}
\caption[]{\small Intra-class variation for ERM  (left) vs. HYPO (right) on PACS. For each class $y$, we measure the Sinkhorn Divergence between the embeddings of each pair of domains. Our method results in significantly lower intra-class variation across different pairs of training domains compared to ERM.}
 \vspace{-0.1cm}
\label{fig:dist_comp}
\end{figure*}

\vspace{-0.2cm}
\paragraph{Additional ablation studies.}
\emph{Due to space constraints, we defer additional experiments and ablations to the Appendix, including \textbf{(1)} results on other tasks from DomainBed (Appendix~\ref{sec:otherood}); \textbf{(2)} results on large-scale benchmarks such as ImageNet-100 (Appendix~\ref{sec:imagenet}); \textbf{(3)} ablation of different loss terms (Appendix~\ref{sec:ablationloss}); \textbf{(4)} an analysis on the effect of $\tau$ and $\alpha$ (Appendix~\ref{sec:hyperparameter})}.

%% file: chap_Arxiv/theory.tex
\section{Why HYPO Improves Out-of-Distribution Generalization?}
\label{sec:theory}
\vspace{-0.2cm}
In this section, we provide a formal justification of the loss function. 
Our main Theorem~\ref{thm:main_theorem} gives a provable understanding of how the learning objective effectively reduces the variation estimate $\Vcal^{\textnormal{sup}}(h, \mathcal{E}_\textnormal{avail})$, thus directly reducing the OOD generalization error according to Theorem~\ref{general ood bound}. 
For simplicity,  we assume $\tau = 1$ and denote the prototype vectors $\boldsymbol{\mu}_1, \ldots, \boldsymbol{\mu}_C \in \mathcal{S}^{d-1}$. Let $\cH \subset \{h: \Xcal \mapsto \mathcal{S}^{d-1} \} $ denote the function class induced by the neural network. 
\begin{theorem}[Variation upper bound using HYPO]\label{thm:main_theorem}
When samples are aligned with class prototypes such that $\frac{1}{N} \sum_{j=1}^N \boldsymbol{\mu}_{c(j)}^\top \*z_j \geq 1-\epsilon$ for some $\epsilon\in (0,1)$, then $\exists\delta\in (0,1)$, with probability at least $1-\delta$, 
\begin{equation*}
    \Vcal^{\textnormal{sup}}(h, \mathcal{E}_\textnormal{avail}) \leq O( {\epsilon^{1/3}} + {(\frac{\ln(2/\delta)}{N})^{1/6}} 
    + {(\bbE_{\mathcal{D}} [\frac{1}{N} \bbE_{\sigma_1, \ldots, \sigma_N}\sup_{h \in \cH } \sum_{i=1}^N \sigma_i \*z_i^\top \boldsymbol{\mu}_{c(i)}])^{1/3}} ),
    \end{equation*}

where $\*z_j = {h(\*x_j)\over \|h(\*x_j)\|_2 } $, $\sigma_1, \ldots, \sigma_N$ are Rademacher random variables and $O(\cdot)$ suppresses dependence on constants and $|\Eava|$. 
\end{theorem}

{\textbf{Implications.} In Theorem~\ref{thm:main_theorem}, we can see that the upper bound consists of three factors: the optimization error, the Rademacher complexity of the given neural network, and the estimation error which becomes close to 0 as the number of samples $N$ increases. Importantly, the term $\epsilon$ reflects how sample embeddings are aligned with their class prototypes on the hyperspherical space (as we have $\frac{1}{N} \sum_{j=1}^N \boldsymbol{\mu}_{c(j)}^{\top} \mathbf{z}_j \geq 1-\epsilon$), \emph{which is directly minimized by our proposed loss in Equation~\ref{eq:loss}}. 
The above Theorem implies that when we train the model with the HYPO loss, we can effectively upper bound the intra-class variation, a key term for bounding OOD generation performance by Theorem~\ref{general ood bound}.  
In Section~\ref{sec:ablationloss}, {we provide empirical verification of our bound by estimating $\hat{\epsilon}$, which is indeed close to 0 for models trained with HYPO loss.} 
We defer proof details to Appendix~\ref{sec:proof}.

\paragraph{Necessity of inter-class separation loss.} {We further present a theoretical analysis in Appendix~\ref{sec:inter-class sep} explaining how our loss promotes inter-class separation, which is necessary to ensure the learnability of the OOD generalization problem.} We provide a brief summary in Appendix~\ref{sec:proof} and discuss the notion of OOD learnability, and would like to refer readers to \cite{ye2021towards} for an in-depth and formal treatment. Empirically, to verify the impact of inter-class separation, we conducted an ablation study in Appendix~\ref{sec:ablationloss}, where we compare the OOD performance of our method (with separation loss) vs. our method (without separation loss). We observe that incorporating separation loss indeed achieves stronger OOD generalization performance, echoing the theory.

%% file: chap_Arxiv/rela.tex
\vspace{-0.2cm}
\section{Related Works}

\vspace{-0.2cm}
\paragraph{Out-of-distribution generalization.} OOD generalization is an important problem when the training and test data are sampled from different distributions. Compared to domain adaptation~\citep{daume2006domain, ben2010theory, tzeng2017adversarial, kang2019contrastive, wang2022embracing}, OOD generalization is more challenging~\citep{blanchard2011generalizing, muandet2013domain, gulrajani2020search, bai2021ood, zhou2021domain, koh2021wilds, bai2021decaug, wang2022generalizing,  ye2022ood, cha2022domain, bai2023improving, kim2023vne,guo2023out, dai2023moderately, tong2023distribution}, which aims to generalize to unseen distributions without any sample from the target domain. In particular, A popular direction is to extract domain-invariant feature representation. Prior works show that the invariant features from training domains can help discover invariance on target domains for linear models~\citep{peters2016causal, rojas2018invariant}. IRM~\citep{arjovsky2019invariant} and its variants~\citep{ahuja2020invariant,krueger2021out} aim to find invariant representation from different training domains via an invariant risk regularizer.
{\color{black} \cite{mahajan2021domain} propose a causal matching-based algorithm for domain generalization.}
Other lines of works have explored the problem from various perspectives such as causal discovery~\citep{chang2020invariant}, distributional robustness~\citep{sagawa2020distributionally, zhou2020learning}, model ensembles~\citep{chen2023explore, rame2023model}, and test-time adaptation~\citep{park2023test, chen2023improved}. In this paper, we focus on improving OOD generalization via \textcolor{black}{hyperspherical} learning, and provide a new theoretical analysis of the generalization error.

\vspace{-0.2cm}
\paragraph{Theory for OOD generalization.} Although the problem has attracted great interest, theoretical
understanding of desirable conditions for OOD generalization is under-explored. Generalization to arbitrary OOD is impossible since the test distribution is unknown~\citep{blanchard2011generalizing, muandet2013domain}. Numerous general distance measures exist for defining a set of test domains around the training domain, such as KL divergence~\citep{joyce2011kullback}, MMD~\citep{gretton2006kernel}, and EMD~\citep{rubner1998metric}. Based on these measures, some prior works focus on analyzing the OOD generalization error bound. For instance, \citet{albuquerque2019generalizing} obtain a risk bound for linear combinations of training domains.
~\citet{ye2021towards} provide OOD generalization error bounds based on the notation of variation. 
In this work, we provide a hyperspherical learning algorithm that provably reduces the variation, thereby improving OOD generalization both theoretically and empirically. 

\vspace{-0.2cm}
{\color{black} 
\paragraph{Contrastive learning for domain generalization} 
Contrastive learning methods have been widely explored in different learning tasks. For example,
~\citet{wang2020understanding} analyze the relation between the alignment and uniformity properties on the hypersphere for unsupervised learning, while we focus on supervised learning with domain shift.
{\color{black}~\citet{tapaswi2019video} investigates a contrastive metric learning approach for hyperspherical embeddings in video face clustering, which differs from our objective of OOD generalization. }
~\citet{von2021self} provide theoretical justification for self-supervised learning with data augmentations. 
Recently, contrastive losses have been adopted for OOD generalization.
For example, CIGA~\citep{chen2022learning} captures the invariance of graphs to enable OOD generalization for graph data. CNC~\citep{zhang2022correct} is specifically designed for learning representations robust to spurious correlation by inferring pseudo-group labels and performing supervised contrastive learning. SelfReg~\citep{kim2021selfreg} proposes a self-supervised contrastive regularization for domain generalization with non-hyperspherical embeddings, while we focus on hyperspherical features with theoretically grounded loss formulations.
}

%% file: chap_Arxiv/conc.tex
\vspace{-0.2cm}
\section{Conclusion}

In this paper, we present a theoretically justified algorithm for OOD generalization via hyperspherical learning. HYPO facilitates learning domain-invariant representations in the hyperspherical space. Specifically, we encourage low variation via aligning features across domains for each class and promote high separation by separating prototypes across different classes. Theoretically, we provide a provable understanding of how our loss function reduces the OOD generalization error. Minimizing our learning objective can reduce the variation estimates, which determine the general upper bound on the generalization error of a learnable OOD generalization task. Empirically, HYPO achieves superior performance compared to competitive OOD generalization baselines. We hope our work can inspire future research on OOD generalization and provable understanding. 

\section*{Acknowledgement}
The authors would like to thank
ICLR anonymous reviewers for their helpful feedback. The work is supported by the AFOSR
Young Investigator Program under award number FA9550-23-1-0184, National Science Foundation
(NSF) Award No. IIS-2237037 \& IIS-2331669, and Office of Naval Research under grant number
N00014-23-1-2643.

%% file: chap_Arxiv/appendix.tex
\section{Pseudo Algorithm} \label{alg} 
The training scheme of HYPO is shown below. We jointly optimize for (1) \emph{low variation}, by encouraging the feature embedding of samples to be close to their class prototypes; and (2) \emph{high separation}, by encouraging different class prototypes to be far apart from each other.

\begin{algorithm}[!h]
{\color{black} 
	\DontPrintSemicolon
	\SetNoFillComment
	\textbf{Input:} Training dataset $\mathcal{D}$, deep neural network encoder $h$, class prototypes $\boldsymbol{\mu}_c$ ($1\le j\le C$), temperature $\tau$\\
	\For {$epoch=1,2,\ldots,$}
	{
	\For {$iter=1,2,\ldots,$}
	{
        sample a mini-batch  $B = \{\*x_i, y_i\}_{i=1}^b$\\
        obtain augmented batch $\tilde{B} = \{\tilde{\*x}_i, \tilde{y}_i\}_{i=1}^{2b}$ by applying two random augmentations to $\*x_i\in B$ $\forall i\in\{1,2,\ldots,b\}$\\
		\For {$\tilde{\*x}_i \in \tilde{B}$}
		{       
		    \tcp{obtain normalized embedding}
          $\tilde{\*z}_i = h(\tilde{\*x}_i)$, $\*z_i = \tilde{\*z}_i / \lVert \tilde{\*z}_i\rVert_2$\\
            \tcp{update class-prototypes}
            ${\boldsymbol{\mu}}_c := \text{Normalize}( \alpha{\boldsymbol{\mu}}_c + (1-\alpha)\*z_i), 
     \; \forall c\in\{1, 2, \ldots, C\} $\\
		}
		\tcp{calculate the loss for low variation}
		      $\mathcal{L}_{\mathrm{var}} =- \frac{1}{N} \sum_{e \in \mathcal{E}_\text{avail}} \sum_{i=1}^{|\mathcal{D}^e|} \log \frac{\exp \left({\*z^e_{i}}^\top  {\boldsymbol{\mu}}_{c(i)} / \tau\right)}{\sum_{j=1}^{C} \exp \left({\*z^e_{i}}^\top  {\boldsymbol{\mu}}_{j} / \tau\right)}$\\
		      \tcp{calculate the loss for high separation}
     $\mathcal{L}_{\mathrm{sep}} =\frac{1}{C} \sum_{i=1}^C \log {1\over C-1}\sum_{\substack{j\neq i, j \in \Ycal}}\exp{\left({\boldsymbol{\mu}}_i^\top {\boldsymbol{\mu}}_j/\tau\right)}$\\
        		\tcp{calculate overall loss}
        $\mathcal{L} =  \mathcal{L}_{\mathrm{var}} +  \mathcal{L}_{\mathrm{sep}}$\\
        \tcp{update the network weights}
        update the weights in the deep neural network \\
	}
	}
\caption{Hyperspherical Out-of-Distribution Generalization}
}
\end{algorithm}

\section{Broader Impacts}

Our work facilitates the theoretical understanding of OOD generalization through prototypical learning, which encourages low variation and high separation in the hyperspherical space. In Section~\ref{exp:results}, we qualitatively and quantitatively verify the low intra-class variation of the learned embeddings and we discuss in Section~\ref{sec:theory} that the variation estimate determines the general upper bound on the generalization error for a learnable OOD generalization task. This provable framework may serve as a foothold that can be useful for future OOD generalization research via representation learning.

From a practical viewpoint, our research can directly impact many real applications, when deploying machine learning models in the real world. Out-of-distribution generalization is a fundamental problem and is commonly encountered when building reliable ML systems in the industry. Our empirical results show that our approach achieves consistent improvement over the baseline on a wide range of tasks. Overall, our work has both theoretical and practical impacts.

\section{Theoretical Analysis}
\label{sec:proof}

\paragraph{Notations.} We first set up notations for theoretical analysis. Recall that $\p_{XY}^e$ denotes the joint distribution of $X,Y$ in domain $e$. The label set $\mathcal{Y} := \{1,2,\cdots, C\}$. For an input $\*x$, $\*z=h(\*x) / \| h(\*x)\|_2$ is its feature embedding. Let $\p^{e,y}_X$ denote the marginal distribution of $X$ in domain $e$ with class $y$. Similarly, $\p^{e,y}_Z$ denotes the marginal distribution of $Z$ in domain $e$ with class $y$. 
Let $E := |\mathcal{E}_\text{train}|$ for abbreviation. As we do not consider the existence of spurious correlation in this work, it is natural to assume that domains and classes are uniformly distributed:
$\p_X := \frac{1}{EC} \sum_{e,y} \p^{e,y}_X$. We specify the distance metric to be the Wasserstein-1 distance \emph{i.e.,} $\cW_1(\cdot, \cdot)$ and define all notions of variation under such distance.

Next, we proceed with several lemmas that are particularly useful to prove our main theorem.

\begin{lemma}\label{lem:rademacher}
With probability at least $1-\delta$,
\begin{align*}
            -\bbE_{(\*x,c)\sim \p_{XY}} \boldsymbol{\mu}_c^\top \frac{h(\*x)}{\norm{h(\*x)}} +  \frac{1}{N}\sum_{i=1}^N {\boldsymbol{\mu}}_{c(i)}^\top\frac{h(\*x_i)}{\norm{h(\*x_i)}}   \leq \bbE_{S \sim \p_N} [\frac{1}{N} \bbE_{\sigma_1, \ldots, \sigma_N} \sup_{h \in \cH } \sum_{i=1}^N \sigma_i \boldsymbol{\mu}_{c(i)}^\top\frac{h(\*x_i)}{\norm{h(\*x_i)}} ] + \beta \sqrt{\frac{\ln(2/\delta)}{N}}.
\end{align*}
where $\beta$ is a universal constant and $\sigma_1, \ldots, \sigma_N$ are Rademacher variables. 
\end{lemma}

\begin{proof}\renewcommand{\qedsymbol}{}
    By Cauchy-Schwarz inequality,
    \begin{align*}
        |\boldsymbol{\mu}_{c(i)}^\top \frac{h(\*x_i)}{\norm{h(\*x_i)}}| \leq \norm{\boldsymbol{\mu}_{c(i)}} \norm{\frac{h(\*x_i)}{\norm{h(\*x_i)}}} = 1
    \end{align*}
    Define $\cG = \{ \ip{\frac{h(\cdot)}{\norm{h(\cdot)}}}{\cdot} : h \in \cH\}$. Let $S = (\*u_1, \ldots, \*u_N)\sim \p_N$ where $\*u_i = \begin{pmatrix}
    \*x_i \\ 
    \boldsymbol{\mu}_{c(i)}
\end{pmatrix}$ and $N$ is the sample size. The Rademacher complexity of $\cG$ is
\begin{align*}
    \cR_N(\cG) := \bbE_{S \sim \p_N} [\frac{1}{N} \sup_{g \in \cG } \sum_{i=1}^N \sigma_i g(\*u_i)].
\end{align*}
We can apply the standard Rademacher complexity bound (Theorem 26.5 in 
\href{https://www.cs.huji.ac.il/~shais/UnderstandingMachineLearning/}{ Shalev-Shwartz and Ben-David}) 
to $\cG$, then we have that, 
    \begin{align*}
            -\bbE_{(\*x,c)\sim \p_{XY}} \boldsymbol{\mu}_c^\top \frac{h(\*x)}{\norm{h(\*x)}} +  \frac{1}{N}\sum_{i=1}^N {\boldsymbol{\mu}}_{c(i)}^\top\frac{h(\*x_i)}{\norm{h(\*x_i)}}   & \leq \bbE_{S \sim \p_N} [\frac{1}{N} \bbE_{\sigma_1, \ldots, \sigma_N} \sup_{g \in \cG } \sum_{i=1}^N \sigma_i g(\*u_i)] + \beta \sqrt{\frac{\ln(2/\delta)}{N}} \\
            & = \bbE_{S \sim \p_N} [\frac{1}{N} \bbE_{\sigma_1, \ldots, \sigma_N} \sup_{h \in \cH } \sum_{i=1}^N \sigma_i \boldsymbol{\mu}_{c(i)}^\top\frac{h(\*x_i)}{\norm{h(\*x_i)}} ] + \beta \sqrt{\frac{\ln(2/\delta)}{N}},
\end{align*}
where $\beta$ is a universal positive constant. 
\end{proof}
\begin{remark}
     The above lemma indicates that when samples are sufficiently aligned with their class prototypes on the hyperspherical feature space, \emph{i.e.},$\frac{1}{N}\sum_{i=1}^N {\boldsymbol{\mu}}_{c(i)}^\top \frac{h(\*x_i)}{\norm{h(\*x_i)}}  \geq 1 - \epsilon$ for some small constant $\epsilon > 0$,  we can upper bound $-\bbE_{(\*x,c)\sim \p_{XY}} \boldsymbol{\mu}_c^\top \frac{h(\*x)}{\norm{h(\*x)}}$. This result will be useful to prove Thm~\ref{thm:main_theorem}.
     \end{remark}

\begin{lemma}\label{lem:sub_classes}
    Suppose $\bbE_{(\*z,c)\sim \p_{ZY}} \boldsymbol{\mu}_c^\top \*z \geq 1-\gamma$. Then, for all $e \in \mathcal{E}_\textnormal{train}$ and $y \in [C]$, we have that
    \begin{align*}
        \bbE_{\*z \sim \p^{e,y}_Z} \boldsymbol{\mu}_c^\top \*z \geq 1 - C E \gamma. 
    \end{align*}
\end{lemma}

\begin{proof}\renewcommand{\qedsymbol}{}
Fix $e^\prime \in \mathcal{E}_\textnormal{train}$ and $y^\prime \in [C]$. Then,
    \begin{align*}
        1-\gamma & \leq \bbE_{(\*z,c)\sim \p_{ZY}} \boldsymbol{\mu}_c^\top \*z \\
        & = \frac{1}{C E} \sum_{e \in \mathcal{E}_\textnormal{train}} \sum_{y \in [C]} \bbE_{\*z \sim \p^{e,y}_Z} \*z^\top \boldsymbol{\mu}_y \\
        & = \frac{1}{C E}  \bbE_{\*z \sim \p^{e^\prime,y^\prime}_Z} \*z^\top \boldsymbol{\mu}_{y^\prime} +    \frac{1}{C E} \sum_{(e, y) \in \mathcal{E}_\textnormal{train} \times [C] \setminus \{(e^\prime, y^\prime)\}}  \bbE_{\*z \sim \p^{e,y}_Z} \*z^\top \boldsymbol{\mu}_y \\
        & \leq \frac{1}{C E}  \bbE_{\*z \sim \p^{e^\prime,y^\prime}_Z} \*z^\top \boldsymbol{\mu}_{y^\prime} + \frac{CE - 1}{CE}
    \end{align*}
    where the last line holds by $|\*z^\top \boldsymbol{\mu}_c | \leq 1$ and we also used the assumption that the domains and classes are uniformly distributed. Rearranging the terms, we have
    \begin{align*}
        1- CE \gamma \leq  \bbE_{\*z \sim \p^{e^\prime,y^\prime}_Z} \*z^\top \boldsymbol{\mu}_{y^\prime}
    \end{align*}
\end{proof}

\begin{lemma}\label{lem:markov}
Fix $y \in [C]$ and $e \in \mathcal{E}_\textnormal{train}$. Fix $\eta > 0$. If $\bbE_{\*z \sim \p^{e,y}_Z}\*z^\top \boldsymbol{\mu}_y \geq 1 - C E \gamma$, then
\begin{align*}
    \p^{e,y}_Z(\norm{\*z - \boldsymbol{\mu}_y} \geq \eta ) \leq \frac{2 C E \gamma}{ \eta^2}.
\end{align*}
\end{lemma}

\begin{proof}\renewcommand{\qedsymbol}{}
    Note that
    \begin{align*}
        \norm{\*z-\boldsymbol{\mu}_y}^2 & = \norm{\*z}^2 + \norm{\boldsymbol{\mu}_y}^2 - 2 \*z^\top \boldsymbol{\mu}_y \\
        & = 2- 2 \*z^\top \boldsymbol{\mu}_y.
    \end{align*}
    Taking the expectation on both sides and applying the hypothesis, we have that
    \begin{align*}
        \bbE_{\*z \sim \p^{e,y}_Z} \norm{\*z-\boldsymbol{\mu}_c}^2 \leq 2 C E \gamma.
    \end{align*}
    Applying Chebyschev's inequality to $\norm{\*z-\boldsymbol{\mu}_y}$, we have that
    \begin{align*}
        \p^{e,y}_Z(\norm{\*z - \boldsymbol{\mu}_y} \geq \eta ) & \leq \frac{\text{Var}(\norm{\*z - \boldsymbol{\mu}_y})}{\eta^2} \\
        & \leq \frac{\bbE_{\*z \sim \p^{e,y}_Z}(\norm{\*z - \boldsymbol{\mu}_y}^2)}{\eta^2} \\
        & \leq \frac{2 C E \gamma}{ \eta^2}
    \end{align*}
\end{proof}

\begin{lemma}\label{lem:wasserstein}
    Fix $y \in [C]$. Fix $e,e^\prime \in  \mathcal{E}_\text{train}$. Suppose $\bbE_{\*z \sim \p^{e,y}_Z}\*z^\top \boldsymbol{\mu}_c \geq 1 - C E \gamma$. Fix $\*v \in S^{d-1}$. Let $P$ denote the distribution of $\*v^\top \*z_e$ and $Q$ denote the distribution $\*v^\top \*z_{e^\prime}$. Then,
    \begin{align*}
        \cW_1(P,Q) \leq 10 (C E \gamma)^{1/3} 
    \end{align*}
    where $\cW_1(P,Q)$ is the Wassersisten-1 distance. 
\end{lemma}

\begin{proof}\renewcommand{\qedsymbol}{}
    Consider the dual formulation of
    \href{https://www.stat.cmu.edu/~larry/=sml/Opt.pdf}{Wasserstein-1 distance}:
    \begin{align*}
        \cW(P,Q) & = \sup_{f : \norml{f} \leq 1} \bbE_{\*x \sim \p^{e,y}_X}[f(\*v^\top \*x)] - \bbE_{\*x \sim \p^{e^\prime,y}_X}[f(\*v^\top \*x)]
    \end{align*}
where $\norml{f}$ denotes the Lipschitz norm. Let $\kappa > 0$. There exists $f_0$ such that 
\begin{align*}
    \cW(P,Q) \leq \bbE_{\*z \sim \p^{e,y}_Z}[f_0(\*v^\top \*z)] - \bbE_{\*z \sim \p^{e^\prime,y}_Z}[f_0(\*v^\top \*z)] + \kappa.
\end{align*}
We assume that without loss of generality $f_0(\boldsymbol{\mu}_y^\top \*v) = 0$. Define $f^\prime(\cdot) = f_0(\cdot) - f_0(\boldsymbol{\mu}_y^\top \*v)$. Then, note that $f^\prime(\boldsymbol{\mu}_y^\top \*v) = 0$ and
\begin{align*}
    \bbE_{\*z \sim \p^{e,y}_Z}[f^\prime(\*v^\top \*z)] - \bbE_{z \sim \p^{e^\prime,y}_Z}[f^\prime(\*v^\top \*z)] & = \bbE_{\*z \sim \p^{e,y}_Z}[f_0(\*v^\top \*z)] - \bbE_{\*z \sim \p^{e^\prime,y}_Z}[f_0(\*v^\top \*z)] + f^\prime(\boldsymbol{\mu}_y^\top v) - f^\prime(\boldsymbol{\mu}_y^\top v) \\
    & = \bbE_{\*z \sim \p^{e,y}_Z}[f_0(\*v^\top \*z)] - \bbE_{\*z \sim \p^{e^\prime,y}_Z}[f_0(\*v^\top \*z)], 
\end{align*}
proving the claim. 

Now define $B := \{\*u \in S^{d-1} : \norm{\*u - \boldsymbol{\mu}_y} \leq \eta \}$. Then, we have
\begin{align*}
    \bbE_{\*z \sim \p^{e,y}_Z}[f_0(\*v^\top \*z)] - \bbE_{\*z \sim \p^{e^\prime,y}_Z}[f_0(\*v^\top \*z)] & = \bbE_{\*z \sim \p^{e,y}_Z}[f_0(\*v^\top \*z) \mathbbm{1}\{\*z \in B\}] - \bbE_{\*z \sim \p^{e^\prime,y}_Z}[f_0(\*v^\top \*z) \mathbbm{1}\{\*z \in B\}] \\
    &+ \bbE_{\*z \sim \p^{e,y}_Z}[f_0(\*v^\top \*z) \mathbbm{1}\{\*z \not \in B\}] - \bbE_{\*z \sim \p^{e^\prime,y}_Z}[f_0(\*v^\top \*z) \mathbbm{1}\{\*z \not \in B\}] 
\end{align*}
Note that if $\*z \in B$, then by $\norml{f} \leq 1$,
\begin{align*}
    |f_0(\*v^\top \*z) - f_0(\*v^\top \boldsymbol{\mu}_y)| & \leq |\*v^\top (\*z - \boldsymbol{\mu}_y)| \\
    & \leq \norm{\*v} \norm{\*z-\boldsymbol{\mu}_y} \\
    & \leq \eta.
\end{align*}
Therefore, $|f_0(\*v^\top \*z)| \leq \eta$ and we have that
    \begin{align*}
        \bbE_{\*z \sim \p^{e,y}_Z}[f_0(\*v^\top \*z) \mathbbm{1}\{\*z \in B\}] - \bbE_{\*z \sim \p^{e^\prime,y}_Z}[f_0(\*v^\top \*z)\mathbbm{1}\{\*z \in B\}] & \leq 2 \eta (\bbE_{\*z \sim \p^{e,y}_Z}[ \mathbbm{1}\{\*z \in B\}] + \bbE_{\*z \sim \p^{e^\prime,y}_Z}[\mathbbm{1}\{\*z \in B\}]) \\
        \leq 2 \eta.
    \end{align*}

Now, note that $\max_{\*u \in S^{d-1} }|f(\*u^\top \*v)| \leq 2$ (repeat the argument from above but use $\norm{\*u - \boldsymbol{\mu}_y} \leq 2$. Then, 
\begin{align*}
    \bbE_{\*z \sim \p^{e,y}_Z}[f_0(\*v^\top \*z) \mathbbm{1}\{\*z \not \in B\}] - \bbE_{\*z \sim \p^{e^\prime,y}_Z}[f_0(\*v^\top \*z) \mathbbm{1}\{\*z \not \in B\}] & \leq 2[\bbE_{\*z \sim \p^{e,y}_Z}[\mathbbm{1}\{\*z \not \in B\}] + \bbE_{\*z \sim \p^{e^\prime,y}_Z}[ \mathbbm{1}\{\*z \not \in B\}]] \\
    & \leq \frac{8 C E \gamma}{\eta}
\end{align*}
where in the last line, we used the hypothesis and Lemma \ref{lem:markov}. Thus, by combining the above, we have that
\begin{align*}
     \cW(P,Q) \leq 2 \eta + \frac{8 C E \gamma}{\eta^2} + \kappa.
\end{align*}
Choosing $\eta = (CE \gamma)^{1/3}$, we have that
\begin{align*}
     \cW(P,Q) \leq 10  (C E \gamma)^{1/3} + \kappa.
\end{align*}
    Since $\kappa >0$ was arbitrary, we can let it go to 0, obtaining the result. 
\end{proof}

Next, we are ready to prove our main results. For completeness, we state the theorem here.
\begin{theorem}[Variation upper bound (Thm 4.1)]\label{thm:main_theorem_app}
Suppose samples are aligned with class prototypes such that $\frac{1}{N} \sum_{j=1}^N \boldsymbol{\mu}_{c(j)}^\top \*z_j \geq 1-\epsilon$ for some $\epsilon\in (0,1)$, where $\*z_j = {h(\*x_j)\over \|h(\*x_j)\|_2 } $. Then $\exists\delta\in (0,1)$, with probability at least $1-\delta$, 
\begin{equation*}
    \Vcal^{\textnormal{sup}}(h, \Sigma_{\textnormal{avail}}) \leq O( 
    {\epsilon^{1/3}} + 
    {(\bbE_{\mathcal{D}} [\frac{1}{N} \bbE_{\sigma_1, \ldots, \sigma_N}\sup_{h \in \cH } \sum_{i=1}^N \sigma_i \*z_i^\top \boldsymbol{\mu}_{c(i)}])^{1/3}} + 
    {(\frac{\ln(2/\delta)}{N})^{1/6}}),
    \end{equation*}

where $\sigma_1, \ldots, \sigma_N$ are Rademacher random variables and $O(\cdot)$ suppresses dependence on constants and $|\Eava|$. 
\end{theorem}
\begin{proof}[Proof of Theorem \ref{thm:main_theorem}]\renewcommand{\qedsymbol}{}

Suppose $\frac{1}{N} \sum_{j=1}^N \boldsymbol{\mu}_{c(j)}^\top \*z_j = \frac{1}{N}\sum_{i=1}^N {\boldsymbol{\mu}}_{c(i)}^\top\frac{h(\*x_i)}{\norm{h(\*x_i)}}   \geq 1 - \epsilon$. Then, by Lemma \ref{lem:rademacher}, with probability {at least} $1-\delta$, we have 
\begin{align*}
  -\bbE_{(\*x,c)\sim \p_{XY}} \boldsymbol{\mu}_c^\top \frac{h(\*x)}{\norm{h(\*x)}}  & \leq \bbE_{S \sim \p_N} [\frac{1}{N} \bbE_{\sigma_1, \ldots, \sigma_N} \sup_{h \in \cH } \sum_{i=1}^N \sigma_i \boldsymbol{\mu}_{c(i)}^\top\frac{h(\*x_i)}{\norm{h(\*x_i)}} ] + \beta \sqrt{\frac{\ln(2/\delta)}{N}} - \frac{1}{N}\sum_{i=1}^N {\boldsymbol{\mu}}_{c(i)}^\top \frac{h(\*x_i)}{\norm{h(\*x_i)}}  \\
  & \leq \bbE_{S \sim \p_N} [\frac{1}{N} \bbE_{\sigma_1, \ldots, \sigma_N} \sup_{h \in \cH } \sum_{i=1}^N \sigma_i \boldsymbol{\mu}_{c(i)}^\top\frac{h(\*x_i)}{\norm{h(\*x_i)}} ] + \beta \sqrt{\frac{\ln(2/\delta)}{N}}  + \epsilon - 1
\end{align*}
where $\sigma_1, \ldots, \sigma_N$ denote Rademacher random variables and $\beta$ is a universal positive constant. Define $\gamma = \epsilon + \bbE_{S \sim \p_N} [\frac{1}{N} \bbE_{\sigma_1, \ldots, \sigma_N}\sup_{h \in \cH } \sum_{i=1}^N \sigma_i \boldsymbol{\mu}_{c(i)}^\top\frac{h(\*x_i)}{\norm{h(\*x_i)}} ] + \beta \sqrt{\frac{\ln(2/\delta)}{N}}$. Then, we have
\begin{align*}
   \bbE_{(\*z,c)\sim \p_{ZY}} \boldsymbol{\mu}_c^\top \*z \geq 1 - \gamma. 
\end{align*}
Then, by Lemma \ref{lem:sub_classes}, for all $e \in \mathcal{E}_\text{train}$ and $y \in [C]$,
\begin{align*}
            \bbE_{\*z \sim \p^{e,y}_Z} \boldsymbol{\mu}_y^\top \*z \geq 1 - C E \gamma. 
\end{align*}

 Let $\alpha > 0$ and $\*v_0$ such that
\begin{align*}
        \mathcal{V}^{\text{sup}}(h, \mathcal{E}_\text{train}) & = \sup_{\*v \in S^{d-1}} \mathcal{V}(\*v^\top h, \mathcal{E}_\text{train}) \leq  \mathcal{V}(\*v_0^\top h, \mathcal{E}_\text{train}) + \alpha
\end{align*}

    Let $Q_{\*v_0}^{e,y}$ denote the distribution of $\*v_0^\top \*z$ in domain $e$ under class $y$. From Lemma \ref{lem:wasserstein}, we have that 
    \begin{align*}
        \cW_1(Q_{\*v_0}^{e,y}, Q_{\*v_0}^{^\prime,y}) \leq 10 (C E \gamma)^{1/3} 
    \end{align*}
    for all $y \in [C]$ and $e,e^\prime \in  \mathcal{E}_\text{train}$.

We have that
\begin{align*}
    \sup_{\*v \in S^{d-1}} \mathcal{V}(\*v^\top h, \mathcal{E}_\text{train}) & = \sup_{\*v \in S^{d-1}} \mathcal{V}(\*v^\top h, \mathcal{E}_\text{train}) \\
    & = \max_y \sup_{e,e^\prime} \cW_1(Q_{\*v_0}^{e,y}, Q_{\*v_0}^{e^\prime,y}) + \alpha \\
    & \leq 10 (C E \gamma)^{1/3} + \alpha.
\end{align*}
Noting that $\alpha$ was arbitrary, we may send it to $0$ yielding 
\begin{align*}
    \sup_{\*v \in S^{d-1}} \mathcal{V}(\*v^\top h, \mathcal{E}_\text{train}) \leq 10 (C E \gamma)^{1/3} . 
\end{align*}

Now, using the inequality that for $a,b,c \geq 0$, $(a+b+c)^{1/3} \leq a^{1/3} + b^{1/3} + c^{1/3}$, we have that
\begin{align*}
    \mathcal{V}^{\text{sup}}(h, \mathcal{E}_\text{train}) \leq O(\epsilon^{1/3} + (\bbE_{S \sim \p_N} [\frac{1}{N} \bbE_{\sigma_1, \ldots, \sigma_N}\sup_{h \in \cH } \sum_{i=1}^N \sigma_i \boldsymbol{\mu}_{c(i)}^\top\frac{h(\*x_i)}{\norm{h(\*x_i)}} ])^{1/3} + \beta (\frac{\ln(2/\delta)}{N})^{1/6})
\end{align*}

\end{proof}

\begin{remark}
 As our loss promotes alignment of sample embeddings with their class prototypes on the hyperspherical space, the above Theorem implies that when such alignment holds, we can upper bound the intra-class variation with three main factors: the optimization error $\epsilon$, the Rademacher complexity of the given neural network, and the estimation error $(\frac{\ln(2/\delta)}{N})^{1/6}$. 
 \end{remark}

\subsection{Extension: From Low Variation to Low OOD Generalization Error}
\citet{ye2021towards}  provide OOD generalization error bounds based on the notation of variation. Therefore, bounding intra-class variation is critical to bound OOD generalization error. For completeness, we reinstate the main results in~\citet{ye2021towards} below, which provide both OOD generalization error upper and lower bounds based on the variation w.r.t. the training domains. Interested readers shall refer to ~\cite{ye2021towards} for more details and illustrations.

\begin{definition}[Expansion Function~\citep{ye2021towards}] \label{def_expan}
We say a function $s:\mathbb R^+ \cup \{0\} \to \mathbb R^+ \cup \{0, +\infty\}$ is an expansion function, iff the following properties hold: 
1) $s(\cdot)$ is monotonically increasing and $s(x)\geq x,\forall x\geq0$; 2) $\lim_{x\to 0^+} s(x) = s(0) = 0$.
\end{definition}

As it is impossible to generalize to an arbitrary distribution, characterizing the relation between $\Eava$ and $\Eall$ is essential to formalize OOD generalization. Based on the notion of expansion function, the learnability of OOD generalization is defined as follows:

\begin{definition}[OOD-Learnability~\citep{ye2021towards}]\label{def_learn}
Let $\Phi$ be the feature space and $\rho$ be a distance metric on distributions. We say an OOD generalization problem from $\Ecal_\textnormal{avail}$ to $\Ecal_\textnormal{all}$ is \emph{learnable} if there exists an expansion function $s(\cdot)$ and $\delta\ge 0$, such that: for all $\phi\in \Phi$\footnote{$\phi$ referred to as feature $h$ in theoretical analysis.} satisfying $\Ical_\rho(\phi,\Eava) \geq \delta$, we have $s(\Vcal_\rho(\phi, \Ecal_\textnormal{avail}))\geq \Vcal_\rho(\phi, \Ecal_\textnormal{all})$. If such $s(\cdot)$ and $\delta$ exist, we further call this problem $(s(\cdot), \delta)$-learnable. 
\end{definition}
For learnable OOD generalization problems, the following two theorems characterize OOD error upper and lower bounds based on variation. 
\begin{theorem}[OOD Error Upper Bound~\citep{ye2021towards}]\label{general ood bound full}
Suppose we have learned a classifier with loss function $\ell(\cdot, \cdot)$ 
such that $\forall e \in \Eall$ and $\forall y \in \Ycal$, $p_{h^e|Y^e} (h|y) \in L^2(\Rbb^d).$
$h(\cdot)\in\mathbb{R}^d$ denotes the feature extractor.
Denote the characteristic function of random variable $h^e|Y^e$ as $\hat p_{h^e|Y^e}(t|y) = \mathbb E [\exp\{i \langle t,h^e \rangle\}|Y^e=y].$ 
Assume the hypothetical space $\mathcal F$ satisfies the following regularity conditions that $\exists \alpha,M_1,M_2 >0, \forall f \in \Fcal, 
\forall e \in \Eall, y\in\Ycal$,
    \begin{align}\label{concentration_assumption}
    \int_{h\in\R^d} p_{h^e|Y^e}(h|y) |h|^\alpha \mathrm d h\leq M_1 \quad \textnormal{and} \quad \int_{t\in\R^d} |\hat p_{h^e|Y^e}(t|y)| |t|^\alpha \mathrm dt \leq M_2.
    \end{align}
If $(\Eava,\Eall)$ is $\big(s(\cdot),\Ical^{\text{inf}}(h,\Eava)\big)$-learnable under $\Phi$ with Total Variation $\rho$\footnote{For two distribution $\mathbb P,\mathbb Q$ with probability density function $p,q$, $\rho(\mathbb P,\mathbb Q) = \frac12\int_x |p(x) - q(x)|\mathrm dx$.}, then we have
\begin{equation}
\label{mainbound}
\err(f) \leq O\Big(s\big(\Vcal^{\textnormal{sup}} (h,\Eava)\big)^{\frac{\alpha^2}{(\alpha+d)^2}}\Big),
\end{equation}
where $O(\cdot)$ depends on $d,C,\alpha,M_1,M_2$.
\end{theorem}

\begin{theorem}[OOD Error Lower Bound~\citep{ye2021towards}]\label{lower bound full}
Consider $0$-$1$ loss: $\ell(\hat y, y) = \mathbb I(\hat y \neq y)$. 
For any $\delta > 0$ and any expansion function satisfying 1) $s'_+(0) \triangleq \lim_{x\to0^+ } \frac {s(x) - s(0)}{x} \in (1,+\infty)$; 2) exists $k>1, t >0$, s.t. $kx \leq s(x) < +\infty, x \in [0,t]$,
there exists a constant $C_0$ and an OOD generalization problem $(\Eava,\Eall)$ that is $(s(\cdot),\delta)$-learnable under linear feature space $\Phi$ w.r.t symmetric KL-divergence $\rho$, s.t. $\forall \varepsilon\in [0,\frac t 2]$, the optimal classifier $f$ satisfying $\Vcal^{\textnormal{sup}}(h,\Eava) = \varepsilon$ will have the OOD generalization error lower bounded by 
\begin{equation}
\err(f)\geq C_0\cdot s(\Vcal^{\textnormal{sup}}(h,\Eava)).
\end{equation}
\end{theorem}

\section{Additional Experimental Details}
\label{sec:expdetails}

\paragraph{Software and hardware.} Our method is implemented with PyTorch 1.10. All experiments are conducted on NVIDIA GeForce RTX 2080 Ti GPUs for small to medium batch sizes and NVIDIA A100 and RTX A6000 GPUs for large batch sizes.

\paragraph{Architecture.} In our experiments, we use ResNet-18 for CIFAR-10, ResNet-34 for ImageNet-100, ResNet-50 for PACS, VLCS, Office-Home and Terra Incognita. Following common practice in prior works~\citep{2020supcon}, we use a non-linear MLP projection head to obtain features in our experiments. The embedding dimension is 128 of the projection head for ImageNet-100. The projection head dimension is 512  for PACS, VLCS, Office-Home, and Terra Incognita.

\paragraph{Additional implementation details.}
In our experiments, we follow the common practice that initializing the network with ImageNet pre-trained weights for PACS, VLCS, Office-Home, and Terra Incognita. We then fine-tune the network for 50 epochs. For the large-scale experiments on ImageNet-100, we fine-tune ImageNet pre-trained ResNet-34 with our method for 10 epochs for computational efficiency. We set the temperature $\tau=0.1$, prototype update factor $\alpha=0.95$ as the default value. We use stochastic gradient descent with momentum $0.9$, and weight decay $10^{-4}$. The search distribution in our experiments for the learning rate hyperparameter is: $\textnormal{lr} \in \{0.005, 0.002, 0.001, 0.0005, 0.0002, 0.0001, 0.00005\}$. The search space for the batch size is $\textnormal{bs} \in \{32, 64\}$. The loss weight $\lambda$ for balancing our loss function ($\mathcal{L}=\lambda\mathcal{L}_\text{var}+\mathcal{L}_\text{sep}$) is selected from $\lambda \in \{1.0, 2.0, 4.0\}$. For multi-source domain generalization, hard negatives can be incorporated by a simple modification to the denominator of the variation loss:

\begin{equation}
     \label{eq:hard_neg}
   \mathcal{L}_\text{var} =  - \frac{1}{N} \sum_{e \in \mathcal{E}_\text{avail}} \sum_{i=1}^{|\mathcal{D}^e|} \log \frac{\exp \left(\*z_{i}^\top  {\boldsymbol{\mu}}_{c(i)} / \tau\right)}{\sum_{j=1}^{C} \exp \left(\*z_{i}^\top  {\boldsymbol{\mu}}_{j} / \tau \right) + \sum_{\substack{j=1} }^N  \mathbb{I}(y_j\neq y_i,e_i = e_j) \exp{\left({\*z}_i^\top  {\*z}_j/\tau\right)} }
\end{equation}

\paragraph{Details of datasets.}
We provide a detailed description of the datasets used in this work:

\textbf{CIFAR-10} \citep{krizhevsky2009learning} is consist of $60,000$ color images with 10 classes. The training set has $50,000$ images and the test set has
$10,000$ images.

\textbf{ImageNet-100} is composed by randomly sampled 100 categories from ImageNet-1K.
This dataset contains the following classes: {\small n01498041, n01514859, n01582220, n01608432, n01616318, n01687978, n01776313, n01806567, n01833805, n01882714, n01910747, n01944390, n01985128, n02007558, n02071294, n02085620, n02114855, n02123045, n02128385, n02129165, n02129604, n02165456, n02190166, n02219486, n02226429, n02279972, n02317335, n02326432, n02342885, n02363005, n02391049, n02395406, n02403003, n02422699, n02442845, n02444819, n02480855, n02510455, n02640242, n02672831, n02687172, n02701002, n02730930, n02769748, n02782093, n02787622, n02793495, n02799071, n02802426, n02814860, n02840245, n02906734, n02948072, n02980441, n02999410, n03014705, n03028079, n03032252, n03125729, n03160309, n03179701, n03220513, n03249569, n03291819, n03384352, n03388043, n03450230, n03481172, n03594734, n03594945, n03627232, n03642806, n03649909, n03661043, n03676483, n03724870, n03733281, n03759954, n03761084, n03773504, n03804744, n03916031, n03938244, n04004767, n04026417, n04090263, n04133789, n04153751, n04296562, n04330267, n04371774, n04404412, n04465501, n04485082, n04507155, n04536866, n04579432, n04606251, n07714990, n07745940}.

\textbf{CIFAR-10-C} is generated based on the previous literature~\citep{hendrycks2018benchmarking}, applying different corruptions on CIFAR-10 data. The corruption types include gaussian noise, zoom blur, impulse noise, defocus blur, snow, brightness, contrast, elastic transform, fog, frost, gaussian blur, glass blur, JEPG compression, motion blur, pixelate, saturate, shot noise, spatter, and speckle noise.

\textbf{ImageNet-100-C} is algorithmically generated with Gaussian noise based on \citep{hendrycks2018benchmarking} for the ImageNet-100 dataset.

\textbf{PACS}~\citep{li2017deeper} is commonly used in OoD generalization.
This dataset contains $9,991$ examples of resolution $224 \times 224$ and four domains with different image styles, namely photo, art painting, cartoon, and sketch with seven categories.

\textbf{VLCS}~\citep{gulrajani2020search} comprises four domains including Caltech101, LabelMe, SUN09, and VOC2007. It contains $10,729$ examples of resolution $224 \times 224$ and 5 classes.

\textbf{Office-Home}~\citep{gulrajani2020search} contains four different domains: art, clipart, product, and real. This dataset comprises $15,588$ examples of resolution $224 \times 224$ and 65 classes.

\textbf{Terra Incognita}~\citep{gulrajani2020search} comprises images of wild animals taken by cameras at four different locations: location100, location38, location43, and location46. This dataset contains $24,788$ examples of
resolution $224 \times 224$ and 10 classes.

\section{Detailed Results on CIFAR-10}
\label{sec:c10_shift}

In this section, we provide complete results of the different corruption types on CIFAR-10. In Table~\ref{tab:fullcorruption}, we evaluate HYPO under various common corruptions. Results suggest that HYPO achieves consistent improvement over the ERM baseline for all 19 different corruptions. 
We also compare our loss (HYPO) with more recent competitive algorithms: EQRM~\citep{eastwood2022probable} and SharpDRO~\citep{huang2023robust}, on the CIFAR10-C dataset (Gaussian noise). The results on ResNet-18 are presented in Table~\ref{tab:cifar10c}.

\begin{table*}[h]
\centering
\caption{Main results for verifying OOD generalization performance on the 19 different covariate shifts datasets. We train on CIFAR-10 as ID, using CIFAR-10-C as the OOD test dataset. Acc. denotes the accuracy on the OOD test set.}
\scalebox{0.79}{
\begin{tabular}{l|cc|cc|cc|cc}
\toprule
\textbf{Method} & \textbf{Corruptions}
& \textbf{Acc.} & \textbf{Corruptions}
& \textbf{Acc.} & \textbf{Corruptions}
& \textbf{Acc.} & \textbf{Corruptions}
& \textbf{Acc.}\\
\midrule
\textbf{CE} & Gaussian noise & 78.09 & Zoom blur & 88.47 & Impulse noise & 83.60 & Defocus blur & 94.85 \\
\textbf{HYPO (Ours)} & Gaussian noise & 85.21 & Zoom blur & 93.28 & Impulse noise & 87.54 & Defocus blur & 94.90 \\
\midrule
\textbf{CE} & Snow & 90.19 & Brightness & 94.83 & Contrast & 94.11 & Elastic transform & 90.36 \\
\textbf{HYPO (Ours)} & Snow & 91.10 & Brightness & 94.87 & Contrast & 94.53 & Elastic transform & 91.64 \\
\midrule
\textbf{CE} & Fog & 94.45 & Frost & 90.33 & Gaussian blur & 94.85 & Glass blur & 56.99 \\
\textbf{HYPO (Ours)} & Fog & 94.57 & Frost & 92.28 & Gaussian blur & 94.91 & Glass blur & 63.66 \\
\midrule
\textbf{CE} & JEPG compression & 86.95 & Motion blur & 90.69 & Pixelate & 92.67 & Saturate & 92.86 \\
\textbf{HYPO (Ours)} & JEPG compression & 89.24 & Motion blur & 93.07 & Pixelate & 93.95 & Saturate & 93.66 \\
\midrule
\textbf{CE} & Shot noise & 85.86 & Spatter & 92.20 & Speckle noise & 85.66 & \textbf{Average} & 88.32 \\
\textbf{HYPO (Ours)} & Shot noise & 89.87 & Spatter & 92.46 & Speckle noise & 89.94 & \textbf{Average} & \textbf{90.56} \\
\bottomrule
\end{tabular}%
}
\label{tab:fullcorruption}%
\end{table*}%

\begin{table*}[t]
\centering
\scalebox{0.73}{\begin{tabular}{lccccc}
\toprule
\textbf{Algorithm}  & \textbf{Art painting} & \textbf{Cartoon} & \textbf{Photo} & \textbf{Sketch} & \textbf{Average Acc. (\%)}\\
\midrule
\textbf{IRM}~\citep{arjovsky2019invariant}  & 84.8 & 76.4 & 96.7 & 76.1 & 83.5  \\
\textbf{DANN}~\citep{ganin2016domain}  & 86.4 & 77.4 & 97.3 & 73.5 & 83.7   \\
\textbf{CDANN}~\citep{li2018deep}  & 84.6 & 75.5 & 96.8 & 73.5 & 82.6  \\
\textbf{GroupDRO}~\citep{sagawa2020distributionally}  & 83.5 & 79.1 & 96.7 & 78.3 & 84.4 \\
\textbf{MTL}~\citep{blanchard2021domain}  & 87.5 & 77.1 & 96.4 & 77.3 & 84.6 \\
\textbf{I-Mixup}~\citep{wang2020heterogeneous,xu2020adversarial,yan2020improve} & 86.1 & 78.9 & 97.6 & 75.8 & 84.6  \\
\textbf{MMD}~\citep{li2018domain}   & 86.1 & 79.4 & 96.6 & 76.5 & 84.7  \\
\textbf{VREx}~\citep{krueger2021out}   & 86.0 & 79.1 & 96.9 & 77.7 & 84.9  \\
\textbf{MLDG}~\citep{li2018learning}  & 85.5 & 80.1 & 97.4 & 76.6 & 84.9  \\
\textbf{ARM}~\citep{zhang2020adaptive} & 86.8 & 76.8 & 97.4 & 79.3 & 85.1 \\
\textbf{RSC}~\citep{huang2020self} & 85.4 & 79.7 & 97.6 & 78.2 & 85.2 \\
\textbf{Mixstyle}~\citep{zhou2021domain}  & 86.8 & 79.0 & 96.6 & 78.5 & 85.2 \\
\textbf{ERM}~\citep{vapnik1999overview}  & 84.7 & 80.8 & 97.2 & 79.3 & 85.5  \\
\textbf{CORAL}~\citep{sun2016deep}  & 88.3 & 80.0 & 97.5 & 78.8 & 86.2 \\
\textbf{SagNet}~\citep{nam2021reducing} & 87.4 & 80.7 & 97.1 & 80.0 & 86.3  \\
\textbf{SelfReg}~\citep{kim2021selfreg} & 87.9 & 79.4 & 96.8 & 78.3 & 85.6 \\
\textbf{GVRT}~\cite{min2022grounding} & 87.9 & 78.4 & 98.2 & 75.7 & 85.1 \\
\textbf{VNE}~\citep{kim2023vne} & 88.6 & 79.9 & 96.7 & 82.3 & 86.9 \\
\textbf{HYPO (Ours)} &   87.2 &  82.3  &  98.0 & 84.5  &  \textbf{88.0}  \\
\bottomrule
\end{tabular}}         
\caption[]{\small Comparison with state-of-the-art methods on the PACS benchmark. All methods are trained on ResNet-50. The model selection is based on a training domain validation set. {To isolate the effect of loss functions, all methods are optimized using standard SGD}. \textcolor{black}{*Results based on retraining of PCL with SGD using official implementation. PCL with SWAD optimization is further compared in Table~\ref{tab:pcl-swad}}. {\color{black}We run HYPO 3 times and report the average and std. $\pm x$ denotes the standard error, rounded to the first decimal point.}
}
\label{tab:pacs}
\end{table*}

\begin{table*}[ht]
\centering
\scalebox{0.72}{\begin{tabular}{lccccc}
\toprule
\textbf{Algorithm}  & \textbf{Art} & \textbf{Clipart} & \textbf{Product} & \textbf{Real World} & \textbf{Average Acc. (\%)}\\
\midrule
\textbf{IRM}~\citep{arjovsky2019invariant}  & 58.9 & 52.2 & 72.1 & 74.0 & 64.3  \\
\textbf{DANN}~\citep{ganin2016domain}  & 59.9 & 53.0 & 73.6 & 76.9 & 65.9   \\
\textbf{CDANN}~\citep{li2018deep}  & 61.5 & 50.4 & 74.4 & 76.6 & 65.7  \\
\textbf{GroupDRO}~\citep{sagawa2020distributionally}  & 60.4 & 52.7 & 75.0 & 76.0 & 66.0 \\
\textbf{MTL}~\citep{blanchard2021domain}  & 61.5 & 52.4 & 74.9 & 76.8 & 66.4 \\
\textbf{I-Mixup}~\citep{wang2020heterogeneous,xu2020adversarial,yan2020improve} & 62.4 & 54.8 & 76.9 & 78.3 & 68.1  \\
\textbf{MMD}~\citep{li2018domain}   & 60.4 & 53.3 & 74.3 & 77.4 & 66.4  \\
\textbf{VREx}~\citep{krueger2021out}   & 60.7 & 53.0 & 75.3 & 76.6 & 66.4  \\
\textbf{MLDG}~\citep{li2018learning}  & 61.5 & 53.2 & 75.0 & 77.5 & 66.8  \\
\textbf{ARM}~\citep{zhang2020adaptive} & 58.9 & 51.0 & 74.1 & 75.2 & 64.8 \\
\textbf{RSC}~\citep{huang2020self} & 60.7 & 51.4 & 74.8 & 75.1 & 65.5 \\
\textbf{Mixstyle}~\citep{zhou2021domain}  & 51.1 & 53.2 & 68.2 & 69.2 & 60.4 \\
\textbf{ERM}~\citep{vapnik1999overview}  & 63.1 & 51.9 & 77.2 & 78.1 & 67.6  \\
\textbf{CORAL}~\citep{sun2016deep}  & 65.3 & 54.4 & 76.5 & 78.4 & 68.7 \\
\textbf{SagNet}~\citep{nam2021reducing} & 63.4 & 54.8 & 75.8 & 78.3 & 68.1  \\
\textbf{SelfReg}~\citep{kim2021selfreg} & 63.6 & 53.1 & 76.9 & 78.1 & 67.9\\
\textbf{GVRT}~\cite{min2022grounding} & 66.3 & 55.8 & 78.2 & 80.4 & 70.1 \\
\textbf{VNE}~\citep{kim2023vne} & 60.4 & 54.7 & 73.7 & 74.7 & 65.9 \\
\textbf{HYPO (Ours)}  & 68.3 & 57.9 & 79.0 & 81.4  &  \textbf{71.7}  \\
\bottomrule
\end{tabular}}         
\caption[]{Comparison with state-of-the-art methods on the Office-Home benchmark. All methods are trained on ResNet-50. The model selection is based on a training domain validation set. To isolate the effect of loss functions, all methods are optimized using standard SGD.
}
\label{tab:officehome}
\end{table*}

\begin{table*}[ht]
\centering
\scalebox{0.7}{\begin{tabular}{lccccc}
\toprule
\textbf{Algorithm}  & \textbf{Caltech101} & \textbf{LabelMe} & \textbf{SUN09} & \textbf{VOC2007} & \textbf{Average Acc. (\%)}\\
\midrule
\textbf{IRM}~\citep{arjovsky2019invariant}  & 98.6 & 64.9 & 73.4 & 77.3 & 78.6 \\
\textbf{DANN}~\citep{ganin2016domain}  & 99.0 & 65.1 & 73.1 & 77.2 & 78.6   \\
\textbf{CDANN}~\citep{li2018deep}  & 97.1 & 65.1 & 70.7 & 77.1 & 77.5  \\
\textbf{GroupDRO}~\citep{sagawa2020distributionally}  & 97.3 & 63.4 & 69.5 & 76.7 & 76.7 \\
\textbf{MTL}~\citep{blanchard2021domain}  & 97.8 & 64.3 & 71.5 & 75.3 & 77.2 \\
\textbf{I-Mixup}~\citep{wang2020heterogeneous,xu2020adversarial,yan2020improve} & 98.3 & 64.8 & 72.1 & 74.3 & 77.4  \\
\textbf{MMD}~\citep{li2018domain}   & 97.7 & 64.0 & 72.8 & 75.3 & 77.5  \\
\textbf{VREx}~\citep{krueger2021out}   & 98.4 & 64.4 & 74.1 & 76.2 & 78.3  \\
\textbf{MLDG}~\citep{li2018learning}  & 97.4 & 65.2 & 71.0 & 75.3 & 77.2  \\
\textbf{ARM}~\citep{zhang2020adaptive} & 98.7 & 63.6 & 71.3 & 76.7 & 77.6 \\
\textbf{RSC}~\citep{huang2020self} & 97.9 & 62.5 & 72.3 & 75.6 & 77.1 \\
\textbf{Mixstyle}~\citep{zhou2021domain}  & 98.6 & 64.5 & 72.6 & 75.7 & 77.9 \\
\textbf{ERM}~\citep{vapnik1999overview}  & 97.7 & 64.3 & 73.4 & 74.6 & 77.5  \\
\textbf{CORAL}~\citep{sun2016deep}  & 98.3 & 66.1 & 73.4 & 77.5 & 78.8 \\
\textbf{SagNet}~\citep{nam2021reducing} & 97.9 & 64.5 & 71.4 & 77.5 & 77.8  \\
\textbf{SelfReg}~\citep{kim2021selfreg} & 96.7 & 65.2 & 73.1 & 76.2 & 77.8 \\
\textbf{GVRT}~\cite{min2022grounding} & 98.8 & 64.0 & 75.2 & 77.9 & 79.0 \\
\textbf{VNE}~\citep{kim2023vne} & 97.5 & 65.9 & 70.4 & 78.4 & 78.1 \\
\textbf{HYPO (Ours)}  & 98.1 & 65.3 & 73.1 & 76.3 & 78.2  \\
\bottomrule
\end{tabular}}         
\caption[]{Comparison with state-of-the-art methods on the VLCS benchmark. All methods are trained on ResNet-50. The model selection is based on a training domain validation set. To isolate the effect of loss functions, all methods are optimized using standard SGD.
}
\label{tab:vlcs}
\end{table*}

\begin{table*}[ht]
\centering
\scalebox{0.66}{\begin{tabular}{lccccc}
\toprule
\textbf{Algorithm}  & \textbf{Location100} & \textbf{Location38} & \textbf{Location43} & \textbf{Location46} & \textbf{Average Acc. (\%)}\\
\midrule
\textbf{IRM}~\citep{arjovsky2019invariant}  & 54.6 & 39.8 & 56.2 & 39.6 & 47.6 \\
\textbf{DANN}~\citep{ganin2016domain}  & 51.1 & 40.6 & 57.4 & 37.7 & 46.7 \\
\textbf{CDANN}~\citep{li2018deep} & 47.0 & 41.3 & 54.9 & 39.8 & 45.8 \\
\textbf{GroupDRO}~\citep{sagawa2020distributionally}  & 41.2 & 38.6 & 56.7 & 36.4 & 43.2 \\
\textbf{MTL}~\citep{blanchard2021domain} & 49.3 & 39.6 & 55.6 & 37.8 & 45.6 \\
\textbf{I-Mixup}~\citep{wang2020heterogeneous,xu2020adversarial,yan2020improve} & 59.6 & 42.2 & 55.9 & 33.9 & 47.9 \\
\textbf{MMD}~\citep{li2018domain}  & 41.9 & 34.8 & 57.0 & 35.2 & 42.2 \\
\textbf{VREx}~\citep{krueger2021out}  & 48.2 & 41.7 & 56.8 & 38.7 & 46.4 \\
\textbf{MLDG}~\citep{li2018learning} & 54.2 & 44.3 & 55.6 & 36.9 & 47.8 \\
\textbf{ARM}~\citep{zhang2020adaptive} & 49.3 & 38.3 & 55.8 & 38.7 & 45.5 \\
\textbf{RSC}~\citep{huang2020self} & 50.2 & 39.2 & 56.3 & 40.8 & 46.6 \\
\textbf{Mixstyle}~\citep{zhou2021domain} & 54.3 & 34.1 & 55.9 & 31.7 & 44.0 \\
\textbf{ERM}~\citep{vapnik1999overview}  & 49.8 & 42.1 & 56.9 & 35.7 & 46.1 \\
\textbf{CORAL}~\citep{sun2016deep}  & 51.6 & 42.2 & 57.0 & 39.8 & 47.7 \\
\textbf{SagNet}~\citep{nam2021reducing} & 53.0 & 43.0 & 57.9 & 40.4 & 48.6 \\
\textbf{SelfReg}~\citep{kim2021selfreg} & 48.8 & 41.3 & 57.3 & 40.6 & 47.0 \\
\textbf{GVRT}~\cite{min2022grounding} & 53.9 & 41.8 & 58.2 & 38.0 & 48.0 \\
\textbf{VNE}~\citep{kim2023vne} & 58.1 & 42.9 & 58.1 & 43.5 & 50.6 \\
\textbf{HYPO (Ours)} & 58.8 & 46.6 & 58.7 & 42.7 & \textbf{51.7}  \\
\bottomrule
\end{tabular}}         
\caption[]{Comparison with state-of-the-art methods on the Terra Incognita benchmark. All methods are trained on ResNet-50. The model selection is based on a training domain validation set. To isolate the effect of loss functions, all methods are optimized using standard SGD.
}
\label{tab:terra}
\end{table*}

\begin{table*}[!h]
\centering
\scalebox{0.85}{\begin{tabular}{lccccc}
\toprule
\textbf{Algorithm}  & \textbf{Art} & \textbf{Clipart} & \textbf{Product} & \textbf{Real World} & \textbf{Average Acc. (\%)}\\
\midrule
\textbf{SWAD}~\citep{cha2021swad}    & 66.1 & 57.7 & 78.4 & 80.2 & 70.6 \\
\textbf{PCL+SWAD}~\citep{yao2022pcl} & 67.3 & 59.9 & 78.7 & 80.7 & 71.6 \\
\textbf{VNE+SWAD}~\citep{kim2023vne} & 66.6 & 58.6 & 78.9 & 80.5 & 71.1 \\
\textbf{HYPO+SWAD (Ours)} & 68.4 & 61.3 & 81.8 & 82.4 & \textbf{73.5} \\
\bottomrule
\end{tabular}}         
\caption[]{Results with SWAD-based optimization on the Office-Home benchmark. 
}
\label{tab:officehome-swad}
\end{table*}

\begin{table*}[!h]
\centering
\scalebox{0.85}{\begin{tabular}{lccccc}
\toprule
\textbf{Algorithm}  & \textbf{Caltech101} & \textbf{LabelMe} & \textbf{SUN09} & \textbf{VOC2007} & \textbf{Average Acc. (\%)}\\
\midrule
\textbf{SWAD}~\citep{cha2021swad}    & 98.8 & 63.3 & 75.3 & 79.2 & 79.1 \\
\textbf{PCL+SWAD}~\citep{yao2022pcl} & 95.8 & 65.4 & 74.3 & 76.2 & 77.9 \\
\textbf{VNE+SWAD}~\citep{kim2023vne} & 99.2 & 63.7 & 74.4 & 81.6 & 79.7 \\
\textbf{HYPO+SWAD (Ours)} & 98.9 & 67.8 & 74.3 & 77.7 & \textbf{79.7} \\
\bottomrule
\end{tabular}}         
\caption[]{Rresults with SWAD-based optimization on the VLCS benchmark. 
}
\label{tab:vlcs-swad}
\end{table*}

\begin{table*}[h]
\centering
\scalebox{0.8}{\begin{tabular}{lccccc}
\toprule
\textbf{Algorithm}  & \textbf{Location100} & \textbf{Location38} & \textbf{Location43} & \textbf{Location46} & \textbf{Average Acc. (\%)}\\
\midrule
\textbf{SWAD}~\citep{cha2021swad}    & 55.4 & 44.9 & 59.7 & 39.9 & 50.0 \\
\textbf{PCL+SWAD}~\citep{yao2022pcl} & 58.7 & 46.3 & 60.0 & 43.6 & 52.1 \\
\textbf{VNE+SWAD}~\citep{kim2023vne} & 59.9 & 45.5 & 59.6 & 41.9 & 51.7 \\
\textbf{HYPO+SWAD (Ours)} & 56.8 & 61.3 & 54.0 & 53.2 & \textbf{56.3} \\
\bottomrule
\end{tabular}}         
\caption[]{Results with SWAD-based optimization on the Terra Incognita benchmark. 
}
\label{tab:terra-swad}
\end{table*}

\section{Additional Evaluations on Other OOD Generalization Tasks}\label{sec:otherood}

In this section, we provide detailed results on more OOD generalization benchmarks, including Office-Home (Table~\ref{tab:officehome}), VLCS (Table~\ref{tab:vlcs}), and Terra Incognita (Table~\ref{tab:terra}). We observe that our approach achieves strong performance on these benchmarks. We compare our method with a collection of OOD generalization baselines such as
\texttt{IRM}~\citep{arjovsky2019invariant}, 
\texttt{DANN}~\citep{ganin2016domain}, 
\texttt{CDANN}~\citep{li2018deep}, 
\texttt{GroupDRO}~\citep{sagawa2020distributionally},
\texttt{MTL}~\citep{blanchard2021domain},
\texttt{I-Mixup}~\citep{zhang2018mixup}, 
\texttt{MMD}~\citep{li2018domain}, 
\texttt{VREx}~\citep{krueger2021out}, 
\texttt{MLDG}~\citep{li2018learning}, 
\texttt{ARM}~\citep{zhang2020adaptive}, 
\texttt{RSC}~\citep{huang2020self},
\texttt{Mixstyle}~\citep{zhou2021domain},
\texttt{ERM}~\citep{vapnik1999overview}, 
\texttt{CORAL}~\citep{sun2016deep}, 
\texttt{SagNet}~\citep{nam2021reducing},
\texttt{SelfReg}~\citep{kim2021selfreg},
\texttt{GVRT}~\cite{min2022grounding},
\texttt{VNE}~\citep{kim2023vne}. These methods are all loss-based and optimized using standard SGD.  On the Office-Home, our method achieves an improved OOD generalization performance of \textbf{1.6}\% compared to a competitive baseline~\citep{sun2016deep}.

We also conduct experiments coupling with SWAD and achieve superior performance on OOD generalization. As shown in Table~\ref{tab:officehome-swad}, Table~\ref{tab:vlcs-swad}, Table~\ref{tab:terra-swad}, our method consistently establish superior results on different benchmarks including VLCS, Office-Home, Terra Incognita, showing the effectiveness of our method via hyperspherical learning.

\newpage

\section{Experiments on ImageNet-100 and ImageNet-100-C}\label{sec:imagenet}

In this section, we provide additional large-scale results on the ImageNet benchmark. We use ImageNet-100 as the in-distribution data and use ImageNet-100-C with Gaussian noise as OOD data in the experiments. In Figure~\ref{fig:imagenet-100}, we observe our method improves OOD accuracy compared to the ERM baseline.

\begin{wrapfigure}[9]{r}{0.4\textwidth}
    \vspace{-1cm}
    \centering
    \includegraphics[width=0.38\textwidth]{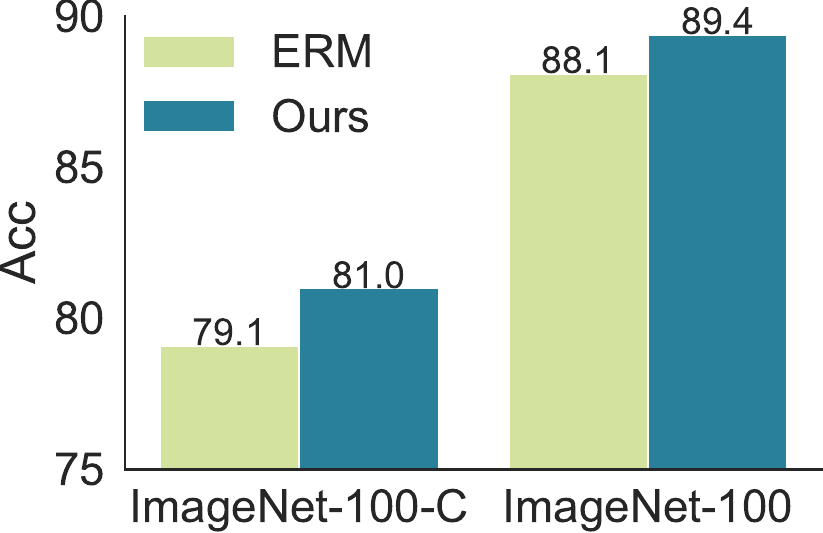}
    \vspace{-0.2cm}
    \caption{Experiments on ImageNet-100 (ID) vs. ImageNet-100-C (OOD).}
    \label{fig:imagenet-100}
\end{wrapfigure}

\section{Ablation of Different Loss Terms}\label{sec:ablationloss}

\paragraph{Ablations on separation loss.} In Table~\ref{tab:ablation-sep}, we demonstrate the effectiveness of the first loss term (variation) empirically. We compare the OOD performance of our method (with separation loss) vs. our method (without separation loss). We observe our method without separation loss term can still achieve strong OOD accuracy--average $87.2\%$ on the PACS dataset. This ablation study indicates the first term (variation) of our method plays a more important role in practice, which aligns with our theoretical analysis in Section~\ref{sec:theory} and Appendix~\ref{sec:proof}.

\begin{table*}[ht]
\centering
\scalebox{0.85}{\begin{tabular}{lccccc}
\toprule
\textbf{Algorithm}  & \textbf{Art painting} & \textbf{Cartoon} & \textbf{Photo} & \textbf{Sketch} & \textbf{Average Acc. (\%)}\\
\midrule
\textbf{Ours (w/o separation loss)} & 86.2 & 81.2 & 97.8 & 83.6 & 87.2 \\
\textbf{Ours (w separation loss)} & 87.2 & 82.3 & 98.0 & 84.5 & \textbf{88.0} \\
\bottomrule
\end{tabular}}         
\caption[]{Ablations on separation loss term. 
}
\label{tab:ablation-sep}
\end{table*}

\paragraph{Ablations on hard negative pairs.} To verify that hard negative pairs help multiple training domains, we conduct ablation by comparing ours (with hard negative pairs) vs. ours (without hard negative pairs). We can see in Table~\ref{tab:hard-negative} that our method with hard negative pairs improves the average OOD performance by $0.4\%$ on the PACS dataset. Therefore, we empirically demonstrate that emphasizing hard negative pairs leads to better performance for multi-source domain generalization tasks.

\begin{table*}[ht]
\centering
\scalebox{0.85}{\begin{tabular}{lccccc}
\toprule
\textbf{Algorithm}  & \textbf{Art painting} & \textbf{Cartoon} & \textbf{Photo} & \textbf{Sketch} & \textbf{Average Acc. (\%)}\\
\midrule
\textbf{Ours (w/o hard negative pairs)} & 87.8  & 82.9  & 98.2  & 81.4  &  87.6  \\
\textbf{Ours (w hard negative pairs)} & 87.2 & 82.3 & 98.0 & 84.5 & \textbf{88.0} \\
\bottomrule
\end{tabular}}         
\caption[]{Ablation on hard negative pairs. OOD generalization performance on the PACS dataset. 
}
\label{tab:hard-negative}
\end{table*}

\paragraph{Comparing EMA update and learnable prototype.}
We conduct an ablation study on the prototype update rule. Specifically, we compare our method with exponential-moving-average (EMA)~\citep{li2020mopro, wang2022contrastive, ming2023cider} prototype update versus learnable prototypes (LP). The results on PACS are summarized in Table~\ref{tab:ema-lp}. We observe our method with EMA achieves better average OOD accuracy $88.0\%$ compared to learnable prototype update rules $86.7\%$. We empirically verify EMA-style method is a suitable prototype updating rule to facilitate gradient-based prototype update in practice.

\begin{table*}[ht]
\centering
\scalebox{0.85}{\begin{tabular}{lccccc}
\toprule
\textbf{Algorithm}  & \textbf{Art painting} & \textbf{Cartoon} & \textbf{Photo} & \textbf{Sketch} & \textbf{Average Acc. (\%)}\\
\midrule
\textbf{Ours (LP)} & 88.0 & 80.7 & 97.5 & 80.7 & 86.7 \\
\textbf{Ours (EMA)}  & 87.2 & 82.3 & 98.0 & 84.5 & \textbf{88.0} \\
\bottomrule
\end{tabular}}         
\caption[]{Ablation on prototype update rules. Comparing EMA update and learnable prototype (LP) on the PACS benchmark. 
}
\label{tab:ema-lp}
\end{table*}

\paragraph{Quantitative verification of the $\epsilon$ factor in Theorem~\ref{thm:main_theorem}.}

We calculate the average intra-class variation over data from all environments $\frac{1}{N} \sum_{j=1}^N \boldsymbol{\mu}_{c(j)}^{\top} \*z_j$ (Theorem~\ref{thm:main_theorem}) models trained with HYPO. Then we obtain $\hat{\epsilon} := 1 - \frac{1}{N} \sum_{j=1}^N \boldsymbol{\mu}_{c(j)}^{\top} \*z_j$. We evaluated PACS, VLCS, and OfficeHome and summarized the results in Table~\ref{tab:quan}. We observe that training with HYPO significantly reduces the average intra-class variation, resulting in a small epsilon ($\hat{\epsilon} < 0.1$) in practice. This suggests that the first term $O(\epsilon^{1\over 3})$ in Theorem~\ref{thm:main_theorem} is indeed small for models trained with HYPO.

\begin{table*}[ht]
\centering
\scalebox{0.85}{\begin{tabular}{lc}
\toprule
\textbf{Dataset}  & $\hat{\epsilon}$ \\
\midrule
\textbf{PACS} & 0.06 \\
\textbf{VLCS}  & 0.08 \\
\textbf{OfficeHome}  & 0.09 \\
\bottomrule
\end{tabular}}         
\caption[]{
Empirical verification of intra-class variation in Theorem~\ref{thm:main_theorem}.
}
\label{tab:quan}
\end{table*}

\begin{table*}[ht]
\centering
\scalebox{0.85}{\begin{tabular}{lc}
\toprule
\textbf{Method}  & OOD Acc. ($\%$) \\
\midrule
\textbf{EQRM}~\citep{eastwood2022probable}  & 77.06 \\
\textbf{SharpDRO}~\citep{huang2023robust} & 81.61 \\
\textbf{HYPO (ours)}  & 85.21 \\
\bottomrule
\end{tabular}}         
\caption[]{
Comparison with more recent competitive baselines. Models are trained on CIFAR-10 using ResNet-18 and tested on CIFAR10-C (Gaussian noise).
}
\label{tab:cifar10c}
\end{table*}

\section{Analyzing the Effect of $\tau$ and $\alpha$}\label{sec:hyperparameter}

In Figure~\ref{fig:alpha}, we present the OOD generalization performance by adjusting the prototype update factor $\alpha$. The results are averaged over four domains on the PACS dataset. We observe the generalization performance is competitive across a wide range of $\alpha$. In particular, our method achieves the best performance when $\alpha=0.95$ on the PACS dataset with an average of $88.0\%$ OOD accuracy.

We show in Figure~\ref{fig:tau} the OOD generalization performance by varying the temperature parameter $\tau$. The results are averaged over four different domains on PACS. We observe a relative smaller $\tau$ results in stronger OOD performance while too large $\tau$ (e.g., $0.9$) would lead to degraded performance.

\begin{figure*}[!ht]
\centering
     \begin{subfigure}[b]{0.36\textwidth}
         \centering
         \includegraphics[width=\textwidth]{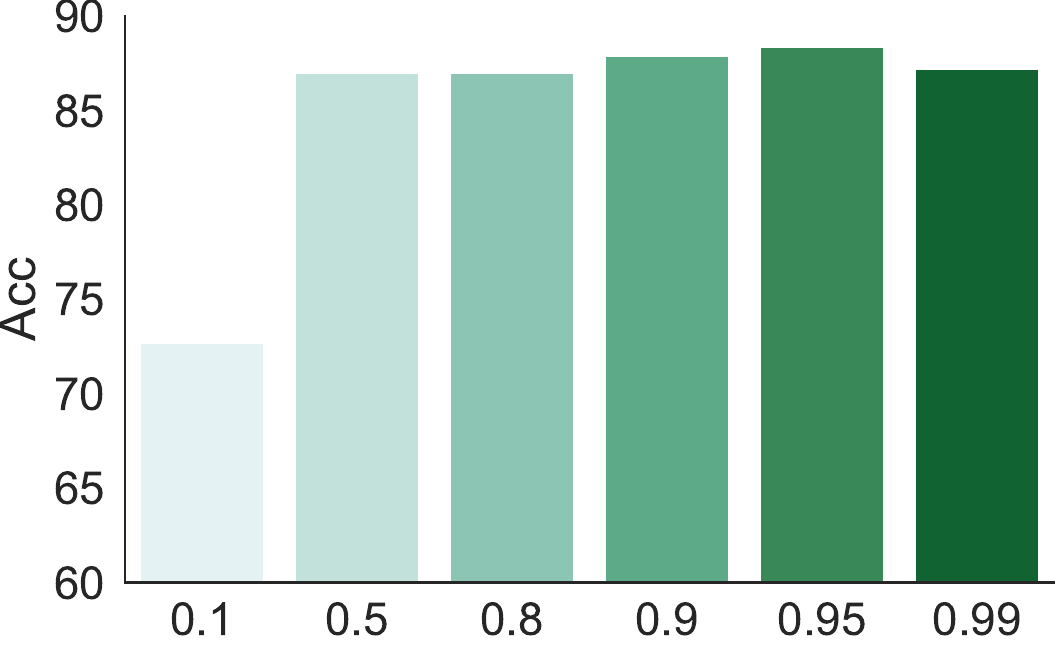}
         \caption{Moving average $\alpha$}
         \label{fig:alpha}
     \end{subfigure}
     \hspace{9mm}
     \begin{subfigure}[b]{0.36\textwidth}
         \centering
         \includegraphics[width=\textwidth]{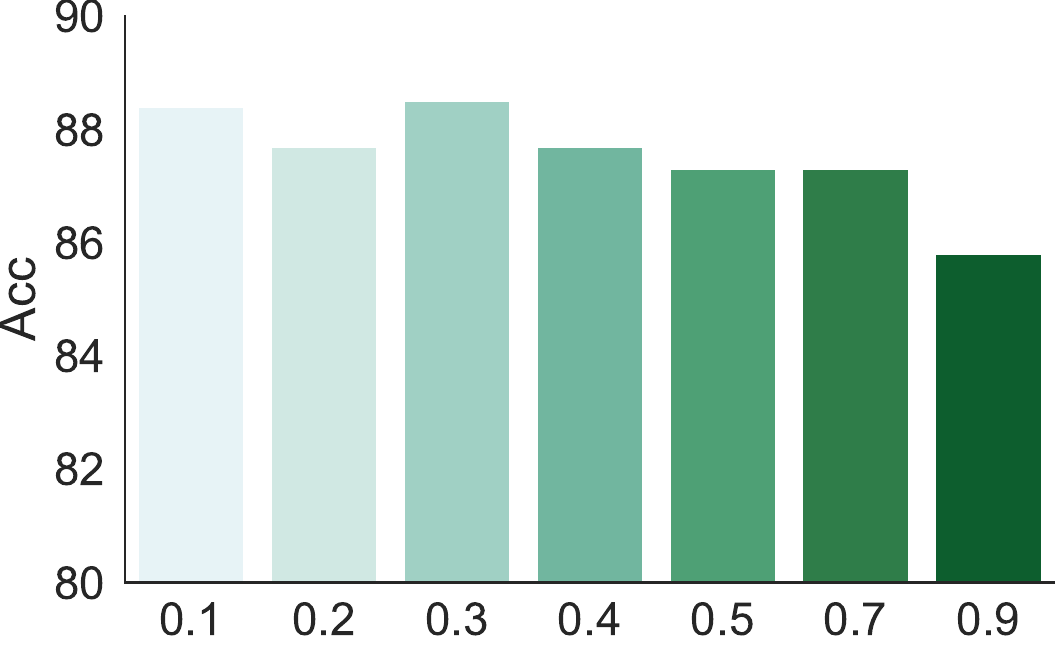}
         \caption{Temperature $\tau$}
         \label{fig:tau}
     \end{subfigure}
\caption[]{Ablation on (a) prototype update discount factor $\alpha$ and (b) temperature $\tau$. The results are averaged over four domains on the PACS dataset.}
\label{fig:alpha-temp}
\end{figure*}

\clearpage

\section{Theoretical Insights on Inter-class Separation}
\label{sec:inter-class sep}
To gain theoretical insights into inter-class separation, we focus on the learned prototype embeddings of the separation loss with a simplified setting where we directly optimize the embedding vectors. 
\begin{definition}
    (Simplex ETF~\citep{sustik2007existence}). A set of vectors $\left\{\boldsymbol{\mu}_i\right\}_{i=1}^C$ in $\mathbb{R}^d$ forms a simplex Equiangular Tight Frame (ETF) if $\left\|\boldsymbol{\mu}_i\right\|=1$ for $\forall i \in[C]$ and $\boldsymbol{\mu}_i^{\top} \boldsymbol{\mu}_j=-1 /(C-1)$ for $\forall i \neq j$.
    \end{definition} 
Next, we will characterize the optimal solution for the separation loss defined as:
\begin{align*}
\mathcal{L}_{\text{sep}} = \underbrace{\frac{1}{C} \sum_{i=1}^C \log \frac{1}{C-1} \sum_{j \neq i, j =1}^C \exp \left(\boldsymbol{\mu}_i^{\top} \boldsymbol{\mu}_j / \tau\right)}_{\uparrow \text { separation }} := \frac{1}{C} \sum_{i=1}^C \log \mathcal{L}_{\text{sep}}(i)
\end{align*}

\begin{lemma} (Optimal solution of the separation loss)
   Assume the number of classes $C \leq d + 1$, $\mathcal{L}_{\text{sep}}$ is minimized when the learned class prototypes $\left\{\boldsymbol{\mu}_i\right\}_{i=1}^C$ form a simplex ETF.
\end{lemma}
\begin{proof} \renewcommand{\qedsymbol}{}
\begin{align} 
\mathcal{L}_{\text{sep}}(i) &= \frac{1}{C-1} \sum_{j \neq i, j =1}^C \exp \left(\boldsymbol{\mu}_i^{\top} \boldsymbol{\mu}_j / \tau\right) \\
&\geq  \exp \left( \frac{1}{C-1} \sum_{j \neq i, j =1}^C \boldsymbol{\mu}_i^{\top} \boldsymbol{\mu}_j / \tau\right) \label{eq:JS} \\
&= \exp \left( {\boldsymbol{\mu}_i^\top\boldsymbol{\mu} - \boldsymbol{\mu}_i^\top\boldsymbol{\mu}_i \over \tau(C-1)}\right) \\
&= \exp \left( {\boldsymbol{\mu}_i^\top\boldsymbol{\mu} - 1 \over \tau(C-1)}\right) 
\end{align}
where we define $\boldsymbol{\mu} = \sum_{i=1}^C \boldsymbol{\mu}_i$ and (\ref{eq:JS}) follows Jensen's inequality. Therefore, we have 
\begin{align*}
\mathcal{L}_{\text{sep}} &=  \frac{1}{C} \sum_{i=1}^C \log \mathcal{L}_{\text{sep}}(i) \\
&\geq \frac{1}{C} \sum_{i=1}^C \log  \exp \left( {\boldsymbol{\mu}_i^\top\boldsymbol{\mu} - 1 \over \tau(C-1)}\right)  \\
&=  { 1 \over \tau C(C-1)} \sum_{i=1}^C (\boldsymbol{\mu}_i^\top\boldsymbol{\mu} - 1)\\
& = { 1 \over \tau C(C-1)} \boldsymbol{\mu}^\top\boldsymbol{\mu} - {1 \over \tau (C-1)}
\end{align*}
It suffices to consider the following optimization problem, 
\begin{align*}
        \text{minimize} \quad & \mathcal{L}_1 =  \boldsymbol{\mu}^\top\boldsymbol{\mu}\\
        \text{subject to} \quad & \|\boldsymbol{\mu}_i\| = 1 \quad \forall i\in[C]
    \end{align*}
where $\boldsymbol{\mu}^\top\boldsymbol{\mu} = (\sum_{i=1}^C\boldsymbol{\mu}_i)^\top(\sum_{i=1}^C\boldsymbol{\mu}_i) = \sum_{i=1}^C\sum_{j\neq i}\boldsymbol{\mu}_i^\top\boldsymbol{\mu}_j +C $

However, the problem is non-convex. We first consider a convex relaxation and show that the optimal solution to the original problem is the same as the convex problem below, 
\begin{align*}
    \text{minimize} \quad & \mathcal{L}_2 = \sum_{i=1}^C \sum_{j=1, j \neq i}^C \boldsymbol{\mu}_i^T \boldsymbol{\mu}_j \\
    \text{subject to} \quad & \|\boldsymbol{\mu}_i\| \leq 1 \quad \forall i\in[C]
\end{align*}
Note that the optimal solution $\mathcal{L}_1^*\geq \mathcal{L}_2^*$. Next, we can obtain the Lagrangian form:
    \begin{equation*}
        \mathcal{L}(\boldsymbol{\mu}_1, \ldots, \boldsymbol{\mu}_C, \lambda_1, \ldots, \lambda_C) = \sum_{i=1}^C \sum_{j=1, j \neq i}^C \boldsymbol{\mu}_i^T \boldsymbol{\mu}_j + \sum_{i=1}^C \lambda_i (\|\boldsymbol{\mu}_i\|^2 - 1)
    \end{equation*}

    where \(\lambda_i\) are Lagrange multipliers. Taking the gradient of the Lagrangian with respect to \(\boldsymbol{\mu}_k\) and setting it to zero, we have:
    \begin{equation*}
        \frac{\partial \mathcal{L}}{\partial \boldsymbol{\mu}_k} = 2 \sum_{i \neq k}^C \boldsymbol{\mu}_i + 2\lambda_k \boldsymbol{\mu}_k = 0
    \end{equation*}
    Simplifying the equation, we have:
    \begin{equation*}
        \boldsymbol{\mu} = \boldsymbol{\mu}_k(1 - \lambda_k)
    \end{equation*}
Therefore, the optimal solution satisfies that (1) either all feature vectors are co-linear (\emph{i.e.} $\boldsymbol{\mu}_k = \alpha_k \boldsymbol{v}$ for some vector $\boldsymbol{v}\in\mathbb{R}^d$ $\forall k\in[C]$) or (2) the sum $\boldsymbol{\mu} = \sum_{i=1}^C \boldsymbol{\mu}_i= \mathbf{0}$. 
The Karush-Kuhn-Tucker (KKT) conditions are:
    \begin{align*}
         \boldsymbol{\mu}_k(1 - \lambda_k)&=\mathbf{0}  \quad \forall k \\
        \lambda_k (\|\boldsymbol{\mu}_k\|^2 - 1) &= 0 \quad \forall k \\
        \lambda_k &\geq 0 \quad \forall k \\
        \|\boldsymbol{\mu}_k\| &\leq 1 \quad \forall k
    \end{align*}
    When the learned class prototypes $\left\{\boldsymbol{\mu}_i\right\}_{i=1}^C$ form a simplex ETF, $\boldsymbol{\mu}_k^\top \boldsymbol{\mu} =1 + \sum_{i\neq k} \boldsymbol{\mu}_i^\top \boldsymbol{\mu}_k = 1 - {C-1 \over C-1} = 0 $. Therefore, we have $\boldsymbol{\mu} = \mathbf{0}$, \( \lambda_k = 1 \), \( \|\boldsymbol{\mu}_k\| = 1 \) and KKT conditions are satisfied. Particularly, \( \|\boldsymbol{\mu}_k\| = 1 \) means that all vectors are on the unit hypersphere and thus the solution is also optimal for the original problem $\mathcal{L}_1$. The solution is optimal for $\mathcal{L}_{\text{sep}}$ as Jensen's inequality (\ref{eq:JS}) becomes equality when $\left\{\boldsymbol{\mu}_i\right\}_{i=1}^C$ form a simplex ETF. The above analysis provides insights on why $\mathcal{L}_{\text{sep}}$ promotes inter-class separation. 

\end{proof}

%% file: main.bbl
\begin{thebibliography}{77}
\providecommand{\natexlab}[1]{#1}
\providecommand{\url}[1]{\texttt{#1}}
\expandafter\ifx\csname urlstyle\endcsname\relax
  \providecommand{\doi}[1]{doi: #1}\else
  \providecommand{\doi}{doi: \begingroup \urlstyle{rm}\Url}\fi

\bibitem[Ahuja et~al.(2020)Ahuja, Shanmugam, Varshney, and Dhurandhar]{ahuja2020invariant}
Kartik Ahuja, Karthikeyan Shanmugam, Kush Varshney, and Amit Dhurandhar.
\newblock Invariant risk minimization games.
\newblock In \emph{International Conference on Machine Learning}, pp.\  145--155, 2020.

\bibitem[Albuquerque et~al.(2019)Albuquerque, Monteiro, Darvishi, Falk, and Mitliagkas]{albuquerque2019generalizing}
Isabela Albuquerque, Jo{\~a}o Monteiro, Mohammad Darvishi, Tiago~H Falk, and Ioannis Mitliagkas.
\newblock Generalizing to unseen domains via distribution matching.
\newblock \emph{arXiv preprint arXiv:1911.00804}, 2019.

\bibitem[Arjovsky et~al.(2019)Arjovsky, Bottou, Gulrajani, and Lopez-Paz]{arjovsky2019invariant}
Martin Arjovsky, L{\'e}on Bottou, Ishaan Gulrajani, and David Lopez-Paz.
\newblock Invariant risk minimization.
\newblock \emph{arXiv preprint arXiv:1907.02893}, 2019.

\bibitem[Bai et~al.(2021{\natexlab{a}})Bai, Sun, Hong, Zhou, Ye, Ye, Chan, and Li]{bai2021decaug}
Haoyue Bai, Rui Sun, Lanqing Hong, Fengwei Zhou, Nanyang Ye, Han-Jia Ye, S-H~Gary Chan, and Zhenguo Li.
\newblock Decaug: Out-of-distribution generalization via decomposed feature representation and semantic augmentation.
\newblock In \emph{Proceedings of the AAAI Conference on Artificial Intelligence}, pp.\  6705--6713, 2021{\natexlab{a}}.

\bibitem[Bai et~al.(2021{\natexlab{b}})Bai, Zhou, Hong, Ye, Chan, and Li]{bai2021ood}
Haoyue Bai, Fengwei Zhou, Lanqing Hong, Nanyang Ye, S-H~Gary Chan, and Zhenguo Li.
\newblock Nas-ood: Neural architecture search for out-of-distribution generalization.
\newblock In \emph{Proceedings of the IEEE/CVF International Conference on Computer Vision}, pp.\  8320--8329, 2021{\natexlab{b}}.

\bibitem[Bai et~al.(2023)Bai, Yang, Xu, Chan, and Zhou]{bai2023improving}
Haoyue Bai, Ceyuan Yang, Yinghao Xu, S-H~Gary Chan, and Bolei Zhou.
\newblock Improving out-of-distribution robustness of classifiers via generative interpolation.
\newblock \emph{arXiv preprint arXiv:2307.12219}, 2023.

\bibitem[Ben-David et~al.(2010)Ben-David, Blitzer, Crammer, Kulesza, Pereira, and Vaughan]{ben2010theory}
Shai Ben-David, John Blitzer, Koby Crammer, Alex Kulesza, Fernando Pereira, and Jennifer~Wortman Vaughan.
\newblock A theory of learning from different domains.
\newblock \emph{Machine Learning}, 79\penalty0 (1):\penalty0 151--175, 2010.

\bibitem[Blanchard et~al.(2011)Blanchard, Lee, and Scott]{blanchard2011generalizing}
Gilles Blanchard, Gyemin Lee, and Clayton Scott.
\newblock Generalizing from several related classification tasks to a new unlabeled sample.
\newblock In \emph{Advances in Neural Information Processing Systems}, volume~24, 2011.

\bibitem[Blanchard et~al.(2021)Blanchard, Deshmukh, Dogan, Lee, and Scott]{blanchard2021domain}
Gilles Blanchard, Aniket~Anand Deshmukh, {\"U}run Dogan, Gyemin Lee, and Clayton Scott.
\newblock Domain generalization by marginal transfer learning.
\newblock \emph{The Journal of Machine Learning Research}, 22\penalty0 (1):\penalty0 46--100, 2021.

\bibitem[Cha et~al.(2021)Cha, Chun, Lee, Cho, Park, Lee, and Park]{cha2021swad}
Junbum Cha, Sanghyuk Chun, Kyungjae Lee, Han-Cheol Cho, Seunghyun Park, Yunsung Lee, and Sungrae Park.
\newblock Swad: Domain generalization by seeking flat minima.
\newblock \emph{Advances in Neural Information Processing Systems}, 34:\penalty0 22405--22418, 2021.

\bibitem[Cha et~al.(2022)Cha, Lee, Park, and Chun]{cha2022domain}
Junbum Cha, Kyungjae Lee, Sungrae Park, and Sanghyuk Chun.
\newblock Domain generalization by mutual-information regularization with pre-trained models.
\newblock In \emph{European Conference on Computer Vision}, pp.\  440--457, 2022.

\bibitem[Chang et~al.(2020)Chang, Zhang, Yu, and Jaakkola]{chang2020invariant}
Shiyu Chang, Yang Zhang, Mo~Yu, and Tommi Jaakkola.
\newblock Invariant rationalization.
\newblock In \emph{International Conference on Machine Learning}, pp.\  1448--1458, 2020.

\bibitem[Chen et~al.(2023{\natexlab{a}})Chen, Zhang, Song, Shan, and Liu]{chen2023improved}
Liang Chen, Yong Zhang, Yibing Song, Ying Shan, and Lingqiao Liu.
\newblock Improved test-time adaptation for domain generalization.
\newblock In \emph{Proceedings of the IEEE/CVF Conference on Computer Vision and Pattern Recognition}, pp.\  24172--24182, 2023{\natexlab{a}}.

\bibitem[Chen et~al.(2020)Chen, Kornblith, Norouzi, and Hinton]{2020simclr}
Ting Chen, Simon Kornblith, Mohammad Norouzi, and Geoffrey Hinton.
\newblock A simple framework for contrastive learning of visual representations.
\newblock In \emph{International Conference on Machine Learning}, 2020.

\bibitem[Chen et~al.(2023{\natexlab{b}})Chen, Hu, Zhou, Li, and Ma]{chen2023explore}
Yimeng Chen, Tianyang Hu, Fengwei Zhou, Zhenguo Li, and Zhi-Ming Ma.
\newblock Explore and exploit the diverse knowledge in model zoo for domain generalization.
\newblock In \emph{International Conference on Machine Learning}, pp.\  4623--4640. PMLR, 2023{\natexlab{b}}.

\bibitem[Chen et~al.(2022)Chen, Zhang, Bian, Yang, Kaili, Xie, Liu, Han, and Cheng]{chen2022learning}
Yongqiang Chen, Yonggang Zhang, Yatao Bian, Han Yang, MA~Kaili, Binghui Xie, Tongliang Liu, Bo~Han, and James Cheng.
\newblock Learning causally invariant representations for out-of-distribution generalization on graphs.
\newblock \emph{Advances in Neural Information Processing Systems}, pp.\  22131--22148, 2022.

\bibitem[Dai et~al.(2023)Dai, Zhang, Fang, Han, and Tian]{dai2023moderately}
Rui Dai, Yonggang Zhang, Zhen Fang, Bo~Han, and Xinmei Tian.
\newblock Moderately distributional exploration for domain generalization.
\newblock In \emph{International Conference on Machine Learning}, 2023.

\bibitem[Daume~III \& Marcu(2006)Daume~III and Marcu]{daume2006domain}
Hal Daume~III and Daniel Marcu.
\newblock Domain adaptation for statistical classifiers.
\newblock \emph{Journal of Artificial Intelligence Research}, 26:\penalty0 101--126, 2006.

\bibitem[Eastwood et~al.(2022)Eastwood, Robey, Singh, Von~K{\"u}gelgen, Hassani, Pappas, and Sch{\"o}lkopf]{eastwood2022probable}
Cian Eastwood, Alexander Robey, Shashank Singh, Julius Von~K{\"u}gelgen, Hamed Hassani, George~J Pappas, and Bernhard Sch{\"o}lkopf.
\newblock Probable domain generalization via quantile risk minimization.
\newblock \emph{Advances in Neural Information Processing Systems}, 35:\penalty0 17340--17358, 2022.

\bibitem[Ganin et~al.(2016)Ganin, Ustinova, Ajakan, Germain, Larochelle, Laviolette, Marchand, and Lempitsky]{ganin2016domain}
Yaroslav Ganin, Evgeniya Ustinova, Hana Ajakan, Pascal Germain, Hugo Larochelle, Fran{\c{c}}ois Laviolette, Mario Marchand, and Victor Lempitsky.
\newblock Domain-adversarial training of neural networks.
\newblock \emph{The Journal of Machine Learning Research}, 17\penalty0 (1):\penalty0 2096--2030, 2016.

\bibitem[Gretton et~al.(2006)Gretton, Borgwardt, Rasch, Sch{\"o}lkopf, and Smola]{gretton2006kernel}
Arthur Gretton, Karsten Borgwardt, Malte Rasch, Bernhard Sch{\"o}lkopf, and Alex Smola.
\newblock A kernel method for the two-sample-problem.
\newblock \emph{Advances in Neural Information Processing Systems}, 19, 2006.

\bibitem[Gulrajani \& Lopez-Paz(2020)Gulrajani and Lopez-Paz]{gulrajani2020search}
Ishaan Gulrajani and David Lopez-Paz.
\newblock In search of lost domain generalization.
\newblock In \emph{International Conference on Learning Representations}, 2020.

\bibitem[Guo et~al.(2023)Guo, Guo, Cao, Wu, and Chang]{guo2023out}
Yaming Guo, Kai Guo, Xiaofeng Cao, Tieru Wu, and Yi~Chang.
\newblock Out-of-distribution generalization of federated learning via implicit invariant relationships.
\newblock \emph{International Conference on Machine Learning}, 2023.

\bibitem[Hendrycks \& Dietterich(2019)Hendrycks and Dietterich]{hendrycks2018benchmarking}
Dan Hendrycks and Thomas Dietterich.
\newblock Benchmarking neural network robustness to common corruptions and perturbations.
\newblock In \emph{International Conference on Learning Representations}, 2019.

\bibitem[Huang et~al.(2020)Huang, Wang, Xing, and Huang]{huang2020self}
Zeyi Huang, Haohan Wang, Eric~P Xing, and Dong Huang.
\newblock Self-challenging improves cross-domain generalization.
\newblock In \emph{European Conference on Computer Vision}, pp.\  124--140, 2020.

\bibitem[Huang et~al.(2023)Huang, Zhu, Xia, Shen, Yu, Gong, Han, Du, and Liu]{huang2023robust}
Zhuo Huang, Miaoxi Zhu, Xiaobo Xia, Li~Shen, Jun Yu, Chen Gong, Bo~Han, Bo~Du, and Tongliang Liu.
\newblock Robust generalization against photon-limited corruptions via worst-case sharpness minimization.
\newblock In \emph{Proceedings of the IEEE/CVF Conference on Computer Vision and Pattern Recognition}, pp.\  16175--16185, 2023.

\bibitem[Izmailov et~al.(2018)Izmailov, Podoprikhin, Garipov, Vetrov, and Wilson]{izmailov2018averaging}
Pavel Izmailov, Dmitrii Podoprikhin, Timur Garipov, Dmitry Vetrov, and Andrew~Gordon Wilson.
\newblock Averaging weights leads to wider optima and better generalization.
\newblock In \emph{Uncertainty in Artificial Intelligence}, 2018.

\bibitem[Joyce(2011)]{joyce2011kullback}
James~M Joyce.
\newblock Kullback-leibler divergence.
\newblock In \emph{International Encyclopedia of Statistical Science}, pp.\  720--722. 2011.

\bibitem[Jupp \& Mardia(2009)Jupp and Mardia]{jupp2009directional}
P.E. Jupp and K.V. Mardia.
\newblock \emph{Directional Statistics}.
\newblock Wiley Series in Probability and Statistics. 2009.
\newblock ISBN 9780470317815.

\bibitem[Kang et~al.(2019)Kang, Jiang, Yang, and Hauptmann]{kang2019contrastive}
Guoliang Kang, Lu~Jiang, Yi~Yang, and Alexander~G Hauptmann.
\newblock Contrastive adaptation network for unsupervised domain adaptation.
\newblock In \emph{IEEE Conference on Computer Vision and Pattern Recognition}, pp.\  4893--4902, 2019.

\bibitem[Khosla et~al.(2020)Khosla, Teterwak, Wang, Sarna, Tian, Isola, Maschinot, Liu, and Krishnan]{2020supcon}
Prannay Khosla, Piotr Teterwak, Chen Wang, Aaron Sarna, Yonglong Tian, Phillip Isola, Aaron Maschinot, Ce~Liu, and Dilip Krishnan.
\newblock Supervised contrastive learning.
\newblock In \emph{Advances in Neural Information Processing Systems}, volume~33, pp.\  18661--18673, 2020.

\bibitem[Kim et~al.(2021)Kim, Yoo, Park, Kim, and Lee]{kim2021selfreg}
Daehee Kim, Youngjun Yoo, Seunghyun Park, Jinkyu Kim, and Jaekoo Lee.
\newblock Selfreg: Self-supervised contrastive regularization for domain generalization.
\newblock In \emph{IEEE International Conference on Computer Vision}, pp.\  9619--9628, 2021.

\bibitem[Kim et~al.(2023)Kim, Kang, Hwang, Shin, and Rhee]{kim2023vne}
Jaeill Kim, Suhyun Kang, Duhun Hwang, Jungwook Shin, and Wonjong Rhee.
\newblock Vne: An effective method for improving deep representation by manipulating eigenvalue distribution.
\newblock In \emph{Proceedings of the IEEE/CVF Conference on Computer Vision and Pattern Recognition}, pp.\  3799--3810, 2023.

\bibitem[Koh et~al.(2021)Koh, Sagawa, Marklund, Xie, Zhang, Balsubramani, Hu, Yasunaga, Phillips, Gao, et~al.]{koh2021wilds}
Pang~Wei Koh, Shiori Sagawa, Henrik Marklund, Sang~Michael Xie, Marvin Zhang, Akshay Balsubramani, Weihua Hu, Michihiro Yasunaga, Richard~Lanas Phillips, Irena Gao, et~al.
\newblock Wilds: A benchmark of in-the-wild distribution shifts.
\newblock In \emph{International Conference on Machine Learning}, pp.\  5637--5664, 2021.

\bibitem[Krizhevsky et~al.(2009)Krizhevsky, Hinton, et~al.]{krizhevsky2009learning}
Alex Krizhevsky, Geoffrey Hinton, et~al.
\newblock Learning multiple layers of features from tiny images.
\newblock \emph{Technical report, University of Toronto}, 2009.

\bibitem[Krueger et~al.(2021)Krueger, Caballero, Jacobsen, Zhang, Binas, Zhang, Le~Priol, and Courville]{krueger2021out}
David Krueger, Ethan Caballero, Joern-Henrik Jacobsen, Amy Zhang, Jonathan Binas, Dinghuai Zhang, Remi Le~Priol, and Aaron Courville.
\newblock Out-of-distribution generalization via risk extrapolation.
\newblock In \emph{International Conference on Machine Learning}, pp.\  5815--5826, 2021.

\bibitem[Li et~al.(2017)Li, Yang, Song, and Hospedales]{li2017deeper}
Da~Li, Yongxin Yang, Yi-Zhe Song, and Timothy~M Hospedales.
\newblock Deeper, broader and artier domain generalization.
\newblock In \emph{IEEE International Conference on Computer Vision}, pp.\  5542--5550, 2017.

\bibitem[Li et~al.(2018{\natexlab{a}})Li, Yang, Song, and Hospedales]{li2018learning}
Da~Li, Yongxin Yang, Yi-Zhe Song, and Timothy Hospedales.
\newblock Learning to generalize: Meta-learning for domain generalization.
\newblock In \emph{AAAI Conference on Artificial Intelligence}, volume~32, 2018{\natexlab{a}}.

\bibitem[Li et~al.(2018{\natexlab{b}})Li, Pan, Wang, and Kot]{li2018domain}
Haoliang Li, Sinno~Jialin Pan, Shiqi Wang, and Alex~C Kot.
\newblock Domain generalization with adversarial feature learning.
\newblock In \emph{IEEE Conference on Computer Vision and Pattern Recognition}, pp.\  5400--5409, 2018{\natexlab{b}}.

\bibitem[Li et~al.(2020)Li, Xiong, and Hoi]{li2020mopro}
Junnan Li, Caiming Xiong, and Steven Hoi.
\newblock Mopro: Webly supervised learning with momentum prototypes.
\newblock In \emph{International Conference on Learning Representations}, 2020.

\bibitem[Li et~al.(2018{\natexlab{c}})Li, Tian, Gong, Liu, Liu, Zhang, and Tao]{li2018deep}
Ya~Li, Xinmei Tian, Mingming Gong, Yajing Liu, Tongliang Liu, Kun Zhang, and Dacheng Tao.
\newblock Deep domain generalization via conditional invariant adversarial networks.
\newblock In \emph{European Conference on Computer Vision}, pp.\  624--639, 2018{\natexlab{c}}.

\bibitem[Mahajan et~al.(2021)Mahajan, Tople, and Sharma]{mahajan2021domain}
Divyat Mahajan, Shruti Tople, and Amit Sharma.
\newblock Domain generalization using causal matching.
\newblock In \emph{International Conference on Machine Learning}, pp.\  7313--7324. PMLR, 2021.

\bibitem[McInnes et~al.(2018)McInnes, Healy, Saul, and Grossberger]{mcinnes2018umap-software}
Leland McInnes, John Healy, Nathaniel Saul, and Lukas Grossberger.
\newblock Umap: Uniform manifold approximation and projection.
\newblock \emph{The Journal of Open Source Software}, 3\penalty0 (29):\penalty0 861, 2018.

\bibitem[Mettes et~al.(2019)Mettes, van~der Pol, and Snoek]{mettes2019hyperspherical}
Pascal Mettes, Elise van~der Pol, and Cees Snoek.
\newblock Hyperspherical prototype networks.
\newblock \emph{Advances in Neural Information Processing Systems}, 32, 2019.

\bibitem[Min et~al.(2022)Min, Park, Kim, Park, and Kim]{min2022grounding}
Seonwoo Min, Nokyung Park, Siwon Kim, Seunghyun Park, and Jinkyu Kim.
\newblock Grounding visual representations with texts for domain generalization.
\newblock In \emph{European Conference on Computer Vision}, pp.\  37--53. Springer, 2022.

\bibitem[Ming et~al.(2023)Ming, Sun, Dia, and Li]{ming2023cider}
Yifei Ming, Yiyou Sun, Ousmane Dia, and Yixuan Li.
\newblock How to exploit hyperspherical embeddings for out-of-distribution detection?
\newblock In \emph{International Conference on Learning Representations}, 2023.

\bibitem[Muandet et~al.(2013)Muandet, Balduzzi, and Sch{\"o}lkopf]{muandet2013domain}
Krikamol Muandet, David Balduzzi, and Bernhard Sch{\"o}lkopf.
\newblock Domain generalization via invariant feature representation.
\newblock In \emph{International Conference on Machine Learning}, pp.\  10--18, 2013.

\bibitem[Nam et~al.(2021)Nam, Lee, Park, Yoon, and Yoo]{nam2021reducing}
Hyeonseob Nam, HyunJae Lee, Jongchan Park, Wonjun Yoon, and Donggeun Yoo.
\newblock Reducing domain gap by reducing style bias.
\newblock In \emph{IEEE Conference on Computer Vision and Pattern Recognition}, pp.\  8690--8699, 2021.

\bibitem[Park et~al.(2023)Park, Han, Kim, and Moon]{park2023test}
Jungwuk Park, Dong-Jun Han, Soyeong Kim, and Jaekyun Moon.
\newblock Test-time style shifting: Handling arbitrary styles in domain generalization.
\newblock In \emph{International Conference on Machine Learning}, 2023.

\bibitem[Peters et~al.(2016)Peters, B{\"u}hlmann, and Meinshausen]{peters2016causal}
Jonas Peters, Peter B{\"u}hlmann, and Nicolai Meinshausen.
\newblock Causal inference by using invariant prediction: identification and confidence intervals.
\newblock \emph{Journal of the Royal Statistical Society}, pp.\  947--1012, 2016.

\bibitem[Rame et~al.(2023)Rame, Ahuja, Zhang, Cord, Bottou, and Lopez-Paz]{rame2023model}
Alexandre Rame, Kartik Ahuja, Jianyu Zhang, Matthieu Cord, L{\'e}on Bottou, and David Lopez-Paz.
\newblock Model ratatouille: Recycling diverse models for out-of-distribution generalization.
\newblock \emph{International Conference on Machine Learning}, 2023.

\bibitem[Rojas-Carulla et~al.(2018)Rojas-Carulla, Sch{\"o}lkopf, Turner, and Peters]{rojas2018invariant}
Mateo Rojas-Carulla, Bernhard Sch{\"o}lkopf, Richard Turner, and Jonas Peters.
\newblock Invariant models for causal transfer learning.
\newblock \emph{The Journal of Machine Learning Research}, 19\penalty0 (1):\penalty0 1309--1342, 2018.

\bibitem[Rubner et~al.(1998)Rubner, Tomasi, and Guibas]{rubner1998metric}
Yossi Rubner, Carlo Tomasi, and Leonidas~J Guibas.
\newblock A metric for distributions with applications to image databases.
\newblock In \emph{International Conference on Computer Vision}, pp.\  59--66, 1998.

\bibitem[Sagawa et~al.(2020)Sagawa, Koh, Hashimoto, and Liang]{sagawa2020distributionally}
Shiori Sagawa, Pang~Wei Koh, Tatsunori~B. Hashimoto, and Percy Liang.
\newblock Distributionally robust neural networks for group shifts: On the importance of regularization for worst-case generalization.
\newblock In \emph{International Conference on Learning Representations}, 2020.

\bibitem[Sun \& Saenko(2016)Sun and Saenko]{sun2016deep}
Baochen Sun and Kate Saenko.
\newblock Deep coral: Correlation alignment for deep domain adaptation.
\newblock In \emph{European Conference on Computer Vision}, pp.\  443--450, 2016.

\bibitem[Sustik et~al.(2007)Sustik, Tropp, Dhillon, and Heath~Jr]{sustik2007existence}
M{\'a}ty{\'a}s~A Sustik, Joel~A Tropp, Inderjit~S Dhillon, and Robert~W Heath~Jr.
\newblock On the existence of equiangular tight frames.
\newblock \emph{Linear Algebra and its applications}, 426\penalty0 (2-3):\penalty0 619--635, 2007.

\bibitem[Tapaswi et~al.(2019)Tapaswi, Law, and Fidler]{tapaswi2019video}
Makarand Tapaswi, Marc~T Law, and Sanja Fidler.
\newblock Video face clustering with unknown number of clusters.
\newblock In \emph{Proceedings of the IEEE/CVF International Conference on Computer Vision}, pp.\  5027--5036, 2019.

\bibitem[Tong et~al.(2023)Tong, Su, Li, Ding, Haoxiang, and Chen]{tong2023distribution}
Peifeng Tong, Wu~Su, He~Li, Jialin Ding, Zhan Haoxiang, and Song~Xi Chen.
\newblock Distribution free domain generalization.
\newblock In \emph{International Conference on Machine Learning}, pp.\  34369--34378. PMLR, 2023.

\bibitem[Tzeng et~al.(2017)Tzeng, Hoffman, Saenko, and Darrell]{tzeng2017adversarial}
Eric Tzeng, Judy Hoffman, Kate Saenko, and Trevor Darrell.
\newblock Adversarial discriminative domain adaptation.
\newblock In \emph{IEEE Conference on Computer Vision and Pattern Recognition}, pp.\  7167--7176, 2017.

\bibitem[Vapnik(1999)]{vapnik1999overview}
Vladimir~N Vapnik.
\newblock An overview of statistical learning theory.
\newblock \emph{IEEE Transactions on Neural Networks}, 10\penalty0 (5):\penalty0 988--999, 1999.

\bibitem[Von~K{\"u}gelgen et~al.(2021)Von~K{\"u}gelgen, Sharma, Gresele, Brendel, Sch{\"o}lkopf, Besserve, and Locatello]{von2021self}
Julius Von~K{\"u}gelgen, Yash Sharma, Luigi Gresele, Wieland Brendel, Bernhard Sch{\"o}lkopf, Michel Besserve, and Francesco Locatello.
\newblock Self-supervised learning with data augmentations provably isolates content from style.
\newblock \emph{Advances in neural information processing systems}, 34:\penalty0 16451--16467, 2021.

\bibitem[Wang et~al.(2022{\natexlab{a}})Wang, Xiao, Li, Feng, Niu, Chen, and Zhao]{wang2022contrastive}
Haobo Wang, Ruixuan Xiao, Yixuan Li, Lei Feng, Gang Niu, Gang Chen, and Junbo Zhao.
\newblock Pico: Contrastive label disambiguation for partial label learning.
\newblock In \emph{International Conference on Learning Representations}, 2022{\natexlab{a}}.

\bibitem[Wang et~al.(2022{\natexlab{b}})Wang, Lan, Liu, Ouyang, Qin, Lu, Chen, Zeng, and Yu]{wang2022generalizing}
Jindong Wang, Cuiling Lan, Chang Liu, Yidong Ouyang, Tao Qin, Wang Lu, Yiqiang Chen, Wenjun Zeng, and Philip Yu.
\newblock Generalizing to unseen domains: A survey on domain generalization.
\newblock \emph{IEEE Transactions on Knowledge and Data Engineering}, 2022{\natexlab{b}}.

\bibitem[Wang et~al.(2022{\natexlab{c}})Wang, Chaudhari, and Davatzikos]{wang2022embracing}
Rongguang Wang, Pratik Chaudhari, and Christos Davatzikos.
\newblock Embracing the disharmony in medical imaging: A simple and effective framework for domain adaptation.
\newblock \emph{Medical Image Analysis}, 76:\penalty0 102309, 2022{\natexlab{c}}.

\bibitem[Wang \& Isola(2020)Wang and Isola]{wang2020understanding}
Tongzhou Wang and Phillip Isola.
\newblock Understanding contrastive representation learning through alignment and uniformity on the hypersphere.
\newblock In \emph{International Conference on Machine Learning}, pp.\  9929--9939, 2020.

\bibitem[Wang et~al.(2020)Wang, Li, and Kot]{wang2020heterogeneous}
Yufei Wang, Haoliang Li, and Alex~C Kot.
\newblock Heterogeneous domain generalization via domain mixup.
\newblock In \emph{IEEE International Conference on Acoustics, Speech and Signal Processing}, pp.\  3622--3626, 2020.

\bibitem[Xu et~al.(2020)Xu, Zhang, Ni, Li, Wang, Tian, and Zhang]{xu2020adversarial}
Minghao Xu, Jian Zhang, Bingbing Ni, Teng Li, Chengjie Wang, Qi~Tian, and Wenjun Zhang.
\newblock Adversarial domain adaptation with domain mixup.
\newblock In \emph{AAAI Conference on Artificial Intelligence}, pp.\  6502--6509, 2020.

\bibitem[Yan et~al.(2020)Yan, Song, Li, Zou, and Ren]{yan2020improve}
Shen Yan, Huan Song, Nanxiang Li, Lincan Zou, and Liu Ren.
\newblock Improve unsupervised domain adaptation with mixup training.
\newblock \emph{arXiv preprint arXiv:2001.00677}, 2020.

\bibitem[Yao et~al.(2022)Yao, Bai, Zhang, Zhang, Sun, Chen, Li, and Yu]{yao2022pcl}
Xufeng Yao, Yang Bai, Xinyun Zhang, Yuechen Zhang, Qi~Sun, Ran Chen, Ruiyu Li, and Bei Yu.
\newblock Pcl: Proxy-based contrastive learning for domain generalization.
\newblock In \emph{IEEE Conference on Computer Vision and Pattern Recognition}, pp.\  7097--7107, 2022.

\bibitem[Ye et~al.(2021)Ye, Xie, Cai, Li, Li, and Wang]{ye2021towards}
Haotian Ye, Chuanlong Xie, Tianle Cai, Ruichen Li, Zhenguo Li, and Liwei Wang.
\newblock Towards a theoretical framework of out-of-distribution generalization.
\newblock \emph{Advances in Neural Information Processing Systems}, 34:\penalty0 23519--23531, 2021.

\bibitem[Ye et~al.(2022)Ye, Li, Bai, Yu, Hong, Zhou, Li, and Zhu]{ye2022ood}
Nanyang Ye, Kaican Li, Haoyue Bai, Runpeng Yu, Lanqing Hong, Fengwei Zhou, Zhenguo Li, and Jun Zhu.
\newblock Ood-bench: Quantifying and understanding two dimensions of out-of-distribution generalization.
\newblock In \emph{IEEE Conference on Computer Vision and Pattern Recognition}, pp.\  7947--7958, 2022.

\bibitem[Zhang et~al.(2018)Zhang, Cisse, Dauphin, and Lopez-Paz]{zhang2018mixup}
Hongyi Zhang, Moustapha Cisse, Yann~N Dauphin, and David Lopez-Paz.
\newblock Mixup: Beyond empirical risk minimization.
\newblock In \emph{International Conference on Learning Representations}, 2018.

\bibitem[Zhang et~al.(2021)Zhang, Marklund, Dhawan, Gupta, Levine, and Finn]{zhang2020adaptive}
Marvin~Mengxin Zhang, Henrik Marklund, Nikita Dhawan, Abhishek Gupta, Sergey Levine, and Chelsea Finn.
\newblock Adaptive risk minimization: Learning to adapt to domain shift.
\newblock In \emph{Advances in Neural Information Processing Systems}, 2021.

\bibitem[Zhang et~al.(2022)Zhang, Sohoni, Zhang, Finn, and Re]{zhang2022correct}
Michael Zhang, Nimit~S Sohoni, Hongyang~R Zhang, Chelsea Finn, and Christopher Re.
\newblock Correct-n-contrast: a contrastive approach for improving robustness to spurious correlations.
\newblock In \emph{International Conference on Machine Learning}, pp.\  26484--26516, 2022.

\bibitem[Zhou et~al.(2020)Zhou, Yang, Hospedales, and Xiang]{zhou2020learning}
Kaiyang Zhou, Yongxin Yang, Timothy Hospedales, and Tao Xiang.
\newblock Learning to generate novel domains for domain generalization.
\newblock In \emph{European Conference on Computer Vision}, pp.\  561--578, 2020.

\bibitem[Zhou et~al.(2021)Zhou, Yang, Qiao, and Xiang]{zhou2021domain}
Kaiyang Zhou, Yongxin Yang, Yu~Qiao, and Tao Xiang.
\newblock Domain generalization with mixstyle.
\newblock In \emph{International Conference on Learning Representations}, 2021.

\bibitem[Zhou et~al.(2022)Zhou, Liu, Qiao, Xiang, and Loy]{zhou2022domain}
Kaiyang Zhou, Ziwei Liu, Yu~Qiao, Tao Xiang, and Chen~Change Loy.
\newblock Domain generalization: A survey.
\newblock \emph{IEEE Transactions on Pattern Analysis and Machine Intelligence}, 2022.

\end{thebibliography}
